\documentclass[twoside]{article}

\usepackage[accepted]{aistats2022}
\usepackage[round]{natbib}

\usepackage{xcolor}  

\usepackage{bm,amsfonts,amssymb,amsthm}
\usepackage{graphicx}
\usepackage{booktabs,multicol, multirow}
\usepackage[T1]{fontenc}  
\usepackage{url}
\usepackage[hidelinks,breaklinks=true]{hyperref}

\def\mPhi{{\bm{\Phi}}}
\def\mDelta{{\bm{\Delta}}}
\def\mA{{\bm{A}}}
\def\mB{{\bm{B}}}
\def\mD{{\bm{D}}}
\def\mF{{\bm{F}}}
\def\mG{{\bm{G}}}
\def\mI{{\bm{I}}}
\def\mK{{\bm{K}}}
\def\mM{{\bm{M}}}
\def\mN{{\bm{N}}}
\def\mP{{\bm{P}}}
\def\mT{{\bm{T}}}
\def\mU{{\bm{U}}}
\def\mQ{{\bm{Q}}}
\def\mW{{\bm{W}}}
\def\mX{{\bm{X}}}
\def\mY{{\bm{Y}}}
\def\mZ{{\bm{Z}}}
\def\va{{\bm{a}}}
\def\vd{{\bm{d}}}
\def\vq{{\bm{q}}}
\def\vv{{\bm{v}}}
\def\vu{{\bm{u}}}
\def\vw{{\bm{w}}}
\def\vx{{\bm{x}}}
\def\vy{{\bm{y}}}
\def\vz{{\bm{z}}}
\def\vmu{{\bm{\mu}}}
\def\vphi{{\bm{\phi}}}
\def\vzero{{\bm{0}}}
\def\vone{{\bm{1}}}

\def\sC{{\mathbb{C}}}
\def\sR{{\mathbb{R}}}

\DeclareMathOperator{\Tr}{Tr}
\newcommand{\Var}{\mathrm{Var}}

\newcommand{\eqnref}[1]{(\ref{#1})}
\newcommand{\gso}{\tilde{\mA}}

\newtheorem{assumption}{Assumption}
\newtheorem{remark}{Remark}
\newtheorem{theorem}{Theorem}
\newtheorem{corollary}[theorem]{Corollary}
\newtheorem{lemma}[theorem]{Lemma}

\begin{document}

%

%
\runningauthor{Seddik, Wu, Lutzeyer \& Vazirgiannis}

\twocolumn[

\aistatstitle{Node Feature Kernels Increase Graph Convolutional Network Robustness}

\aistatsauthor{ Mohamed El Amine Seddik\\ Huawei, Paris, France
    \And Changmin Wu\\ LIX, École Polytechnique, \\Institute Polytechnique de Paris, France 
    \AND  Johannes F. Lutzeyer\\ LIX, École Polytechnique, \\Institute Polytechnique de Paris, France 
    \And Michalis Vazirgiannis\\LIX, École Polytechnique, \\Institute Polytechnique de Paris, France;\\AUEB, Athens, Greece}
\aistatsaddress{}
]

\begin{abstract}
The robustness of the much used Graph Convolutional Networks (GCNs) to perturbations of their input is becoming a topic of increasing importance. In this paper the \textit{random} GCN is introduced for which a random matrix theory analysis is possible. This analysis suggests that if the graph is sufficiently perturbed, or in the extreme case random, then the GCN fails to benefit from the node features. It is furthermore observed that enhancing the message passing step in GCNs by adding the node feature kernel to the adjacency matrix of the graph structure solves this problem. An empirical study of a GCN utilised for node classification on six real datasets further confirms the theoretical findings and demonstrates that perturbations of the graph structure can result in GCNs performing significantly worse than Multi-Layer Perceptrons run on the node features alone. In practice, adding a node feature kernel to the message passing of perturbed graphs results in a significant improvement of the GCN's performance, thereby rendering it more robust to graph perturbations. Our code is publicly available at: \href{https://github.com/ChangminWu/RobustGCN}{https://github.com/ChangminWu/RobustGCN}. 
\end{abstract}

\section{INTRODUCTION}
In recent years Graph Neural Networks (GNNs) have been a highly impactful model type for the analysis of graph data. This is mainly due to their dominating empirical performance and ability to process attributed graphs composed of node information and an underlying graph structure. Many GNN architectures have been proposed, successively improving on weaknesses of previous architectures (e.g. \citet{corso2020principal, hamilton2017inductive, Xu2019}). 
A popular GNN architecture which has remained a benchmark throughout the past years, partly due to the simplicity of its model equation and partly due to its good performance is the Graph Convolutional Network (GCN) \citep{Kipf2017}. The GCN is part of a class of GNNs called message passing neural networks \citep{Gilmer2017}, where the computations are split into a message passing step in which node features are aggregated over neighbourhoods in the underlying graph structure and an update step in which node features are processed, most commonly by a Multi-Layer Perceptron (MLP). 

While much work is being done in the empirical exploration of GNNs, 
relatively fewer advances have been made in their theoretical analysis. A major advance in the theoretical line of research was the expressivity analysis of different message passing operators performed by \citet{Xu2019} and \citet{Morris2019}. This analyses inspired many researchers to further investigate the expressivity of GNNs and resulted in a multitude of new architectures being proposed \citep{Maron2019, Dasoulas2020}. Another upcoming topic in the development of GNNs is their robustness to perturbations of the underlying graph structure \citep{Zuegner2019, Sun2020}. 

In the presented work, we introduce the \textit{random} GCN, in which parameters of the update step are randomly sampled from Gaussian distributions rather than trained as is commonly the case. The \textit{random} GCN allows us to make use of several powerful random matrix theory tools to gain a theoretical understanding of the factors driving the inference obtained from the GCN model. Our most insightful hypothesis obtained in this way is that \textit{the message passing step dilutes (or in the extreme case completely ignores) information present in the node features if the underlying graph structure is noisy (or in the extreme case completely random)}. In our theoretical analysis we observe that if information of the node features is introduced to the message passing operation, then this loss of information is avoided. This leads us to hypothesise that the addition of the node feature kernel to the message passing operators in GNNs could render them more robust to noise or mispecification of the underlying graph structure. 

In a second part of our presented work we test the hypotheses, obtained in our study of the  \textit{random} GCN, on the state-of-the-art GCN architecture applied to six real-world benchmark datasets. This allows us to empirically verify our theoretical insight, rendering the random features approach for theoretical analysis a promising avenue for further theoretical study of GNN architectures, and the inclusion of node feature information in the message passing step a valid method to increase the robustness of GNNs. 

Our main findings may be summarised as: \textbf{(i)} We contribute both a theoretical and an empirical understanding of how graph and node feature information is processed by the GCN, and \textbf{(ii)} importantly find that the preservation of node feature information is entirely dependent on an informative underlying graph structure. \textbf{(iii)} We furthermore, propose a novel GCN message passing scheme which results in more robust inference from a GCN to structural noise.  

The remainder of this paper is organised as follows. In Section \ref{sec:related_work} we introduce related literature. In Section \ref{sec:theory} we propose the  \textit{random} GCN and analyse it using tools from random matrix theory. The theoretical insight from Section \ref{sec:theory} is then empirically verified in Section \ref{sec:experiments}, where we confirm our hypotheses on the standard GCN on six benchmark datasets and observe the robust performance of the GCN when the node feature kernel matrix is added in the message passing step.

\section{RELATED WORK} \label{sec:related_work}

There exists an extensive literature branch which studies \textit{adversarial attack and defence strategies on graph data} in the context of GNNs summarised in \citet{Guennemann2022},  \citet{Sun2020} and \citet{Zhou2020} with the latter pointing out directly the need for the development of more robust GNNs. In this paper we present one approach to robustifying the performance of GNNs to graph perturbations. In this literature the focus often lies on specific attack strategies perturbing the graph structure in order to alter the inference obtained from a GNN, most commonly the GCN, and defence strategies which aim to develop methodology which is robust to these attacks. 
Recent advances in this literature include, \citet{Zuegner2019} proposing a meta learning approach to find optimal graph perturbations. Their perturbation mechanism is found to drastically decrease the global performance of GNNs to be in some cases worse than simple benchmarks such as logistic regression run on the node features only. \citet{Zuegner2020} propose an algorithm which certifies robustness of individual nodes for the GCN used for node classification under perturbations of the graph structure.
In \citet{Geisler2020} and \citet{Jin2021} the message passing operator in the GCN is replaced by the Soft Mediod function and the sum of several distance based adjacency matrices, respectively, with the aim of more robust GCN performance. 
\citet{Jin2020} propose to learn the graph structure jointly with the GNN parameters to robustify performance and also \citet{Entezari2020} propose to alter the graph structure by using a low rank approximation of the adjacency matrix. The works of \citet{Zhu2019} and \citet{Zhang2020} are most closely related to our proposition of using a node feature kernel to reweight edges in Section \ref{sec:kernel-theory} as they both propose to reweight edges based on the node features. Our theoretical findings in Section \ref{sec:theory} support this approach of more directly taking the node features into account in the aggregation scheme of GNNs to increasing their robustness.






This paper distinguishes itself from adversarial attacks and defence literature fundamentally in that we study untargeted, random graph perturbations which arise as a result of mispecification of the data or uncertainty in the recording methods of the networks. For this kind of perturbation we are able to provide both theoretical (on a toy data example) and empirical understanding, which enables us to offer a distinction between the node feature data and the graph data in networks data sets and how these different information sources are processed by a GNN architecture. 

Our work is also related to the literature studying the challenges that heterophilic graphs pose for GNNs \citep{Pei2020,Zhu2020,Zhu2021}. This literature distinguishes homophilic and heterophilic graphs, in which edges in the graph predominantly connect nodes of equal and unequal classes, respectively. Both of these structures can be, from a theoretical standpoint, equally class-informative, it is only the structure of the class-information which varies. In our work here we consider an orthogonal problem, which is the situation of a diminishing class-structure in the graph, independent of its homo- or heterophilic nature, and the effect this diminishment has on the ability of GNNs to process the information contained in the node features.  


\section{ANALYSIS OF THE RANDOM GCN} \label{sec:theory}

In this section we present our theoretical analysis and main findings. 
Throughout this section $\Vert \cdot \Vert$ denotes the Euclidean (resp., spectral) norm for vectors (resp., matrices); $\Vert \cdot \Vert_F$ denotes the Frobenius norm.
Specifically, we consider a \textit{random} GCN model\footnote{In Section \ref{sec:result}, we show that the performance of the large \emph{random} GCN matches that of the vanilla GCN.}, defined as
\begin{align}\label{eq:model}
    \mPhi = \sigma(\gso \mX \mW),
\end{align}
where $\gso\in \sR^{n\times n}$ denotes the normalised adjacency operator encoding the graph structure (see \eqnref{eq:laplacian} for its  definition), $\mX\in \sR^{n\times p}$ corresponds to the node features, $\mW\in \sR^{p\times d}$ is a random matrix with $W_{ij} \sim \mathcal{N}(0, 1)$ independent and identically distributed (i.i.d.) and $\sigma$ is an activation function applied entry-wise. 
 In particular, we will 
study the spectral behaviour of the \textit{Gram matrix}\footnote{$\mG$ provides access to the internal functioning and performance evaluation of the random GCN.} defined as
\begin{align}\label{eq:Gram}
    \mG = \frac{1}{d} \mPhi \mPhi^\intercal = \frac{1}{d} \sigma(\gso \mX \mW ) \sigma(\mW^\intercal \mX^\intercal \gso^\intercal  ).
\end{align}
To analyse $\mG$ we require assumptions on the node features and graph structure.
\begin{assumption}[Node features]\label{ass:GMM} We suppose that $\mX^\intercal = [\vx_1, \ldots, \vx_n]\in \sR^{p\times n},$ where $\vx_1, \ldots, \vx_n$ are independent node feature vectors, each being a sample from one of $k=2$ distribution classes $\mathcal{C}_1$ and $\mathcal{C}_2$. We further assume that the node feature vectors $\vx_i$ follow a Gaussian mixture model; Specifically, for $\vx_i \in \mathcal{C}_a$, $\vx_i = (-1)^a\frac{\vmu}{\sqrt p} + \vz_i$ for some vector $\vmu \in \sR^p $ and $\vz_i\sim \mathcal{N}(\vzero, \mI_p/p)$.
\end{assumption}
We stress that Assumption \ref{ass:GMM} can be relaxed to a larger class of random vectors $\vx\in \mathcal{X},$ where $\mathcal{X}$ denotes any normed space, satisfying the concentration property $\mathbb P ( \vert \varphi(\vx) - \mathbb E [\varphi(\vx)] \vert > t) \leq C e^{-(t/\sigma)^q}$ with $q \in \sR^+$, for all $1$-Lipschitz functions $\varphi:\mathcal{X}\to \sR$. Such vectors are called \textit{random concentrated vectors} and have the particular property to be stable by Lipschitz transformations \citep{louart2018concentration}. 
The simplest example of concentrated vectors is the standard Gaussian vector $\vz \sim \mathcal{N}(\vzero, \mI_p)$ \citep{LED05}. A more complicated class of examples arises from the fact that the concentration property is stable through Lipschitz maps: if $\vz\in \sR^d$ is concentrated and $g:\sR^d\to \sR^p$ is $1$-Lipschitz, then $g(\vz)$ is also concentrated. A large family of \textit{generative models} falls under this more complicated class of examples, such as, the ``fake'' images generated by Generative Adversarial Networks due to these images being  constructed as Lipschitz transformations of random Gaussian vectors \citep{seddik2020random}.

Now we introduce the underlying model that defines the graph structure. 
We assume that the adjacency matrix $\mA$ of the graph is generated by a stochastic block model \citep{karrer2011stochastic}.
\begin{assumption}[Graph structure] \label{ass:SBM}We assume that the entries of $\mA$ are independent (except for $A_{ii}=1$ for all $i$) Bernoulli random variables with parameter $\pi_{ij} = q^2C_{ab} \in (0, 1)$ for $\vx_i\in \mathcal{C}_a$ and $\vx_j \in \mathcal{C}_b$. In particular, $q\in (0, 1)$ represents the probability 
of an edge occurring between two nodes, 
while $C_{ab}$ represents the probability of an edge arising between 
nodes in classes $\mathcal{C}_a$ and $\mathcal{C}_b.$ 
\end{assumption}

Note that self-loops are implicitly added in Assumption 2, where we assume $A_{ii}=1$ for all $i.$ 
Therefore, we consider that the normalised adjacency operator in \eqnref{eq:model} is defined as
\begin{align}\label{eq:laplacian}
    \gso = \frac{1}{\sqrt n} \left( \mA - \vq \vq^\intercal \right),
\end{align}
where $\vq = q \vone_n$\footnote{The vectors $\vq$ can be consistently estimated through the degree vector $\vd = \mA \vone_n$ as $\vq \approx \vd / \sqrt{\vd^\intercal \vone_n}$.}. The centring by $\vq \vq^\intercal$ is necessary for the eigenvectors corresponding to the extremal eigenvalues of the operator to be class informative \citep{Li2018}, i.e., the centering operation removes the uninformative eigenvector corresponding to the largest eigenvalue of the adjacency matrix, simplifying the theoretical analysis. Specifically, for our analysis in the asymptotic regime where $n\to \infty$ (see Assumption \ref{ass:Growth_rate} subsequently), the centring with $\vq \vq^\intercal$ and the normalisation by $\frac{1}{\sqrt{n}}$ are required so that $\gso$ has a bounded spectral norm asymptotically. In practice, the centring by $\vq \vq^\intercal$ is not feasible as it results in a dense matrix. In our experiments in Section \ref{sec:experiments}, we see this discrepancy to be of little consequence in practice. 


\begin{remark}
Assumption \ref{ass:SBM} allows us to sample directed as well as undirected graphs. Often the spectral analysis of graphs needs to be restricted to undirected graphs, since the analysis of complex-valued spectra arising for directed graphs poses a significant challenge. We are able to include directed graphs since the Gram matrix, analysed in Section \ref{subsec:gram}, and $\tilde\mX \tilde\mX^\intercal,$ analysed in Section \ref{subsec:yyt}, have real spectra even if the underlying graph structure is directed.
\end{remark}

\subsection{Spectral Behaviour of the Gram Matrix} \label{subsec:gram}
Let $\tilde \mX^\intercal = \left[\tilde\vx_1, \ldots, \tilde\vx_n \right] = \mX^\intercal \gso \in \sR^{p\times n}$, the entries of the Gram matrix defined in \eqnref{eq:Gram} are given by
\begin{align*}
    G_{ij} = \frac{1}{d} \sigma(\mW^\intercal \tilde \vx_i)^\intercal \sigma(\mW^\intercal \tilde \vx_j) = \frac{1}{d} \sum_{\ell = 1}^d \sigma(\vw_\ell^\intercal \tilde \vx_i)  \sigma(\vw_\ell^\intercal \tilde \vx_j),
\end{align*}
where $\vw_\ell^\intercal$ denotes the $\ell$-th row of $\mW^\intercal.$ Since all the $\vw_\ell$ follow the same distribution $\mathcal{N}(\vzero, \mI_p)$, taking the expectation over $\vw \sim \mathcal{N}(\vzero, \mI_p)$ (conditionally on $\mX$ and $\mA$) yields the average Gram matrix $\bar \mG$ defined with entries 
\begin{align}
    \bar \mG_{ij} = \mathbb E_{\vw|\mX, \mA}\left[ \sigma(\vw^\intercal \tilde \vx_i)  \sigma(\vw^\intercal \tilde \vx_j)\right].
\end{align}
In particular, in the large $n,p,d$ limit, it has been shown in \citep{louart2018random} that the spectrum (and largest eigenvectors) of $\mG$ are fully described by $\bar \mG$. Specifically, the \textit{resolvent} of $\mG$ defined as, 
\begin{align}
    \mQ(z) = \left( \mG + z \mI_n \right)^{-1},
\end{align}
for $z\in \sC_+$ (with $\Im(z)>0$), has a \textit{deterministic equivalent}\footnote{Such a deterministic equivalent is a standard object within random matrix theory \citep{hachem2007deterministic} since it allows us to characterise the behaviour of the eigenvalues of $\mG$ as well as its largest (often informative) eigenvectors. Specifically, the \textit{spectral measure} $\mu_n = \frac1n \sum_{i=1}^n\delta_{\lambda_i(\mG)}$ of $\mG$ (where $\lambda_i(\mG)$ denotes the $i^{\mathrm{th}}$ eigenvalue of $\mG$) is related to $\mQ(z)$ through the \textit{Stieltjes transform} $q_n(z)=\int(t-z)^{-1}\mu_n(dt) = \frac1n \Tr( \mQ(-z))$. While the eigenvector $\hat\vu_i\in \sR^n$ corresponding to eigenvalue $\lambda_i(\mG)$ is related to $\mQ(z)$ through the Cauchy-integral $\hat{\vu}_i\hat{\vu}_i^\intercal = \frac{-1}{2\pi i} \oint_{\Gamma_i} \mQ(-z) dz$ where $\Gamma_i$ is a positively oriented complex contour surrounding $\lambda_i(\mG)$.} $\bar \mQ(z)$ (conditionally on $\mX$ and $\mA$). 
In other words, for all $\mM \in \sR^{n\times n}$ and $\vu, \vv\in \sR^n$ of bounded spectral and Euclidean norms, respectively, with probability one, $\frac1n \Tr \left( \mM(\mQ(z) - \bar \mQ(z))\right) \to 0, \quad \vu^\intercal (\mQ(z) - \bar \mQ(z)) \vv \to 0,$ which we will simply express using the notation $\mQ(z) \leftrightarrow \bar \mQ(z)$.\par 
A large dimensional growth rate assumption provides the existence of $\bar \mQ(z).$ 
\begin{assumption}[Growth rate]\label{ass:Growth_rate} As $n\to \infty$, \textbf{1.} $p/n\to c\in (0, \infty)$ and $d/n\to r\in (0, \infty)$; \textbf{2.} $\limsup_n \Vert \tilde \mX \Vert < \infty$\footnote{This assumption holds if additional assumptions on the node feature mean vector $\vmu$ and the graph parameters $C_{ab}$, which shall be provided Assumption \ref{ass:conditions}, are placed.} and $\vert \mathcal{C}_a \vert/n\to c_a\in (0, 1)$; \textbf{3.} $\sigma$ is $\lambda_\sigma$-Lipschitz continuous with $\lambda_\sigma>0$ constant.
\end{assumption}
Under Assumption \ref{ass:Growth_rate}, we have from \citep{louart2018random}
\begin{align}\label{eq:Q_bar}
    \mQ(z) \leftrightarrow \bar \mQ(z) = \left( \frac{\bar \mG}{1+\delta_g(z)} + z \mI_n \right)^{-1},
\end{align}
where $\delta_g(z)$ is the unique positive solution to the fixed point equation $\delta_g(z) = \frac1n \Tr (\bar \mG \bar \mQ(z))$.\par
From \eqnref{eq:Q_bar}, to describe the behaviour of $\mG$ one needs to address the non-linearity $\sigma$ in the matrix $\bar \mG$, this is achieved by approximating $\bar \mG$ by a more tractable form in the large $n$ limit. An additional regularity condition on $\sigma$ is needed which we formulate now. 
\begin{assumption}[Regularity of $\sigma$]\label{ass:regularity} Suppose that $\sigma$ is twice differentiable with $\limsup_{n,x\in \sR} \vert \sigma''(x) \vert < \infty$. Furthermore, for $\xi\sim \mathcal{N}(0, 1)$ suppose $\mathbb E [\sigma(\xi)]=0$ and $\mathbb E [\sigma^2(\xi)]=1$.
\end{assumption}
Denote the quantity $b_\sigma = \mathbb E[\sigma'(\xi)]$. Under Assumptions \ref{ass:Growth_rate}-\ref{ass:regularity}, from \citep[Lemma F.1]{fan2020spectra}, the average Gram matrix $\bar \mG$ can be approximated by the $n\times n$ matrix $\tilde{\mG} = b_\sigma^2 \tilde \mX \tilde \mX^\intercal + (1-b_\sigma^2) \mI_n$, since almost surely as $n\to \infty$
\begin{align}\label{eq:approx}
    \frac1n \Vert \bar \mG - \tilde{\mG} \Vert^2_F \to 0.
\end{align}
This approximation ensures in particular that $\bar \mG$ and $\tilde{\mG}$ share the same spectrum.
\begin{remark} The approximation of  $\bar\mG$ by $\tilde{\mG}$ in \eqnref{eq:approx} is valid when the matrix $\tilde \mX \tilde \mX^\intercal$ is of bounded spectral norm. This will be ensured in Assumption \ref{ass:conditions} where additional assumptions are placed on our model parameters $\vmu$ and $C_{ab}$. Furthermore, since the node features $\vx_i$ follow a Gaussian mixture model (as per Assumption \ref{ass:GMM}), if $\gso$ has a bounded spectral norm, then the matrix $\tilde \mX$ falls under the setting of \citep{fan2020spectra} in which the relation in \eqnref{eq:approx} holds.
\end{remark}
Since the behaviour of the average Gram matrix $\bar \mG$ reduces to the analysis of the spectral behaviour of the matrix $\tilde\mX \tilde\mX^\intercal$ as per the approximation in \eqnref{eq:approx}, we will analyse $\tilde\mX \tilde\mX^\intercal$ for the remainder of Section \ref{sec:theory}.

\subsection{Spectral Behaviour of $\tilde X \tilde X^\intercal$} 
\label{subsec:yyt}
We first need further controls on the quantities $\vmu$ and $C_{ab}$ as we describe in the following assumption.
\begin{assumption}\label{ass:conditions} As $n\to \infty$, \textbf{1.} $\limsup_n \Vert \vmu \Vert < \infty$; \textbf{2.} $C_{aa}=1 + \frac{\eta_a}{\sqrt n}$ for $a\in \{1, 2\}$ and $C_{ab}=1$ for $a\neq b \in \{1, 2\}$, where $\eta_a = (-1)^a\eta$ and $\limsup_n \eta < \infty$. 
\end{assumption}
\begin{remark}
Assumption \ref{ass:conditions}.2 defines a dense graph such that the clustering with spectral methods is not asymptotically trivial. Real-World graphs are usually sparse and fall within our theoretical analysis by considering the entry-wise multiplication of the adjacency matrix $\mA$ by a random binary mask as is done by \citet{zarrouk2020performance}. Furthermore without loss of generality, we have specified $\eta_a = (-1)^a\eta$ for clarity of exposition of our theoretical results (Theorem \ref{thm:main}), which can be generalised to different choices of the inter-class similarities (choices of $\eta_a$).
\end{remark}
\begin{figure*}[t!]
\centering
\includegraphics[width=.305\linewidth]{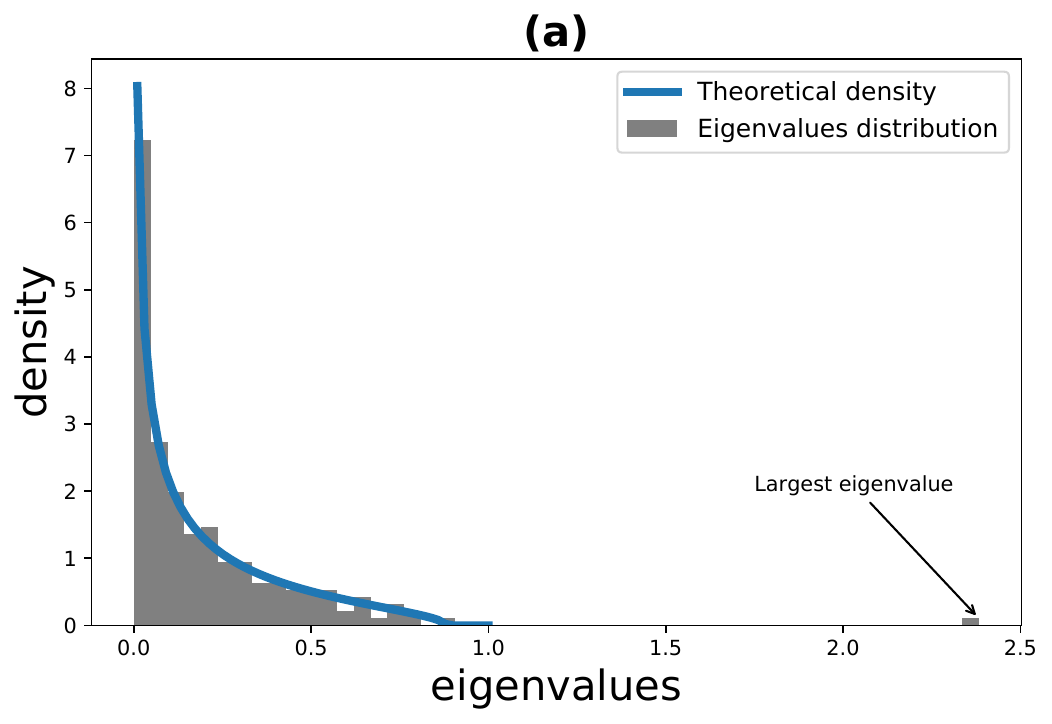}
\includegraphics[width=.32\linewidth]{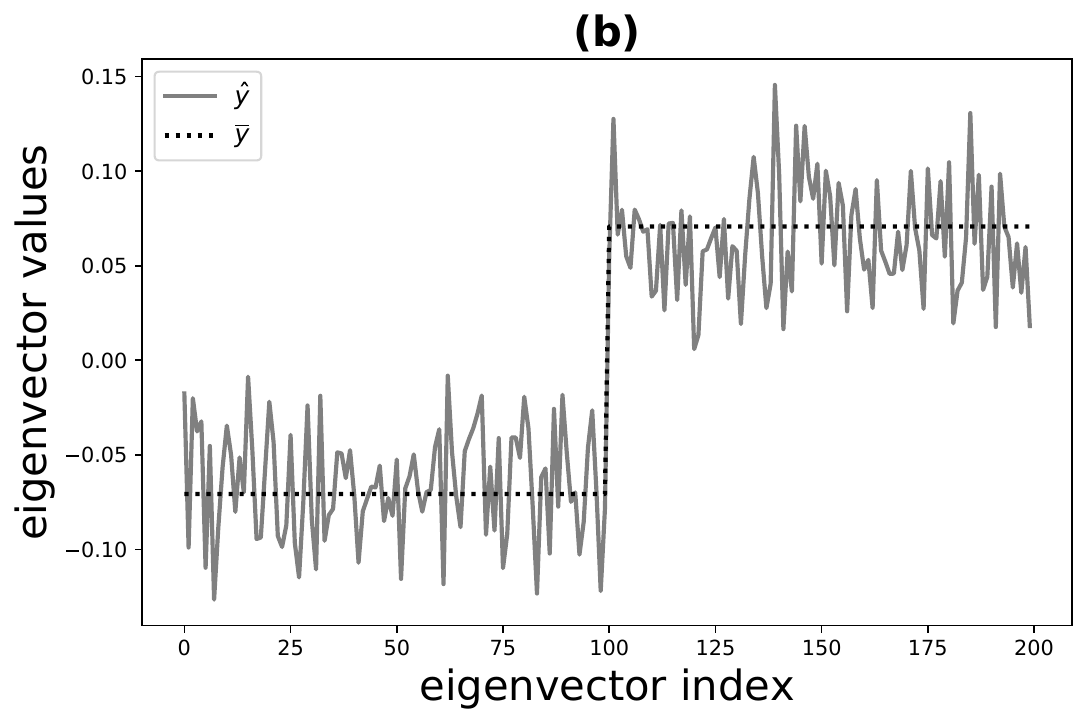}
\includegraphics[width=.31\textwidth]{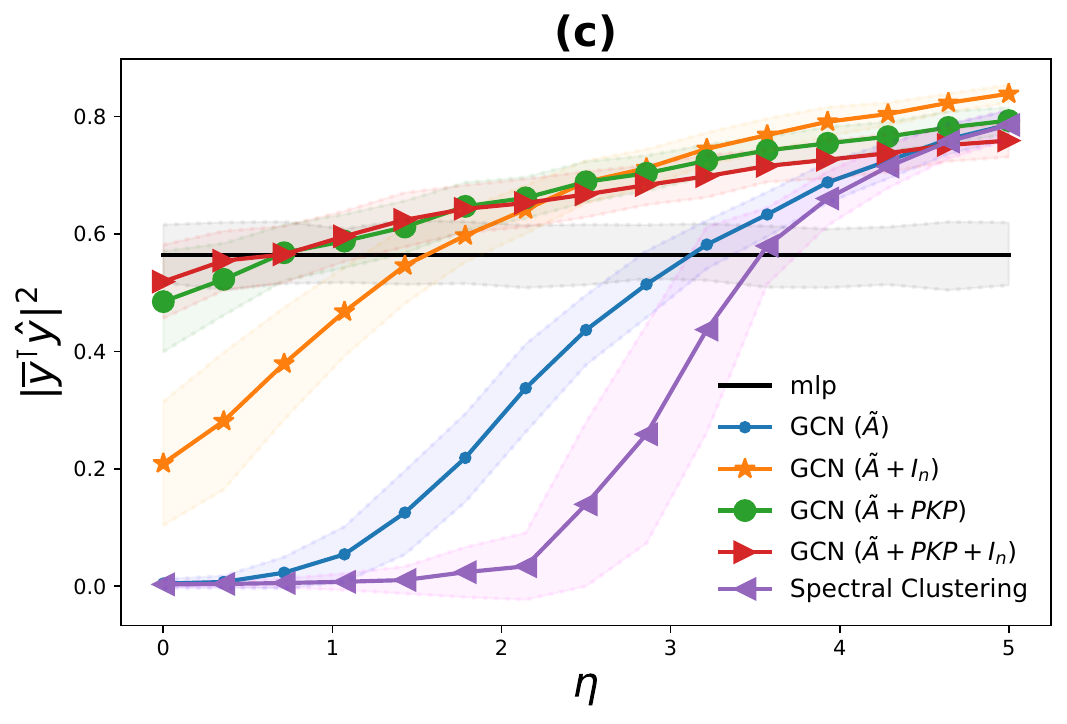}
\caption{\textbf{(a)} Eigenvalues distribution of $\tilde \mX\tilde \mX^\intercal$ versus the theoretical density as per Theorem \ref{thm:main} (the theoretical density is obtained as $f(x)=\frac{1}{\pi}\lim_{\epsilon\to 0} \Im [q(x+i\epsilon)]$ where $q(z)=\frac1n \Tr(\bar \mQ_{\tilde \mX}(z))$). \textbf{(b)} Eigenvector of $\tilde \mX\tilde \mX^\intercal$ corresponding to its largest eigenvalue which correlates with $\bar \vy$. The parameters are: $p=1000$, $n=200$, $q=0.5$, $\eta=4$ and $\vmu=[2,\vzero_{p-1}]^\intercal$. \textbf{(c)} Alignment between the largest eigenvector of $\tilde \mX\tilde \mX^\intercal$ and the labels vector $\bar \vy$ for different added node feature kernel message passing strategies in terms of $\eta$. The different parameters are: $p=500$, $n=250$, $q=0.4$, $\vmu=[1.7,\vzero_{p-1}]^\intercal$ and the kernel matrix has entries $K_{ij}=\vx_i^\intercal \vx_j,$ 
mean and std computed over $100$ runs. \textit{\textbf{The GCN with message passing operator $\tilde\mA + \mP\mK\mP$ outperforms other models when the graph structure is noisy (i.e., low values of $\eta$).}}}\label{fig:abc}
    \vspace{-.3cm}
\end{figure*}
Our main 
result (Theorem \ref{thm:main}) provides a deterministic equivalent for the resolvent of $\tilde \mX\tilde\mX^\intercal$ defined as
\begin{align}
    \mQ_{\tilde\mX}(z) = \left( \tilde\mX \tilde\mX^\intercal + z\mI_n \right)^{-1}.
\end{align}
\begin{theorem}\label{thm:main} Define the quantities $\gamma_f=\Vert \vmu \Vert^2$, $\gamma_g=q^2\eta$, $\nu = q^2(1-q^2)$ and the matrices $\mU = \begin{bmatrix}
    \bar \vy & \vphi
    \end{bmatrix} \in \sR^{n\times 2}$,
\begin{align*}
    \mB = \begin{bmatrix}
    \gamma_g^2(\frac{\gamma_f}{c}+1) & \gamma_g(\frac{\gamma_f}{c}+1)\\
    \gamma_g(\frac{\gamma_f}{c}+1) & \frac{\gamma_f}{c}
    \end{bmatrix},\quad
    \mT = \begin{bmatrix}
    1 & 0\\
    0 & \nu
    \end{bmatrix},
\end{align*}
where $\bar \vy =\frac{\vy}{\sqrt n}$ (with $\vy\in\{-1, 1\}^n$ the vector of labels) and $\vphi = \frac{1}{\sqrt n} \mN\bar\vy$ with $\mN \in \sR^{n\times n}$ a random matrix having random \textit{i.i.d.\@} entries with zero mean and variance $\nu$. Under Assumptions \ref{ass:GMM}, \ref{ass:SBM}, \ref{ass:Growth_rate} and \ref{ass:conditions}, the resolvent $\mQ_{\tilde\mX}(z)$ has a deterministic\footnote{The matrix $\bar \mQ_{\tilde\mX}(z)$ is not deterministic since it depends on the random vector $\vphi$. However, since we are interested in evaluating quantities of the forms $\frac1n \Tr(\mM\bar\mQ_{\tilde\mX}(z))$ or $\vu^\intercal \bar\mQ_{\tilde\mX}(z) \vv$ for $\mM$, $\vu$ and $\vv$ independent of $\vphi$, $\bar \mQ_{\tilde\mX}(z)$ has a deterministic behaviour asymptotically as $n\to \infty$. } equivalent defined as
\begin{align*}
    \bar\mQ_{\tilde\mX}(z) = \zeta \cdot (1+\delta_1) \left( \mI_n - \zeta \mU \left[ \mB^{-1} + \zeta \mT \right]^{-1}\mU^\intercal \right),
\end{align*}
where $\zeta = \frac{1+\delta_2}{\nu+z(1+\delta_1)(1+\delta_2)}$ and $(\delta_1, \delta_2)$ is the unique couple solution of the fixed point equations system
{\small\begin{align*}
    \delta_1 = \frac1c \frac{\nu (1+\delta_1)}{\nu + z(1+\delta_1)(1+\delta_2)},\, \delta_2 = \frac{\nu (1+\delta_2)}{\nu + z(1+\delta_1)(1+\delta_2)}.
\end{align*}}
\end{theorem}
\begin{proof}[Sketch of proof] The proof starts by determining a random equivalent of the adjacency matrix $\mA$. Since $A_{ij}$ is Bernoulli distributed (see Assumption \ref{ass:SBM}) with parameter $q^2(1  + (-1)^{k_i}\delta_{k_i=k_j} \eta /\sqrt n )$ with $k_i\in \{1, 2\}$ the class of node $i$, we may write $A_{ij}=q^2+q^2(-1)^{k_i}\delta_{k_i=k_j} \eta /\sqrt n + N_{ij} $ where $N_{ij}$ is a zero mean random variable with variance $\nu+\mathcal{O}(n^{-\frac12})$. Hence, $\Vert \gso - (q^2\eta \bar\vy \bar \vy^\intercal + \frac{1}{\sqrt n} \mN ) \Vert \to 0$ as $n\to \infty$. Finally, exploiting standard random matrix theory tools from \citep{hachem2007deterministic, louart2018concentration} provides the deterministic equivalent $\bar \mQ_{\tilde\mX}(z)$.
\end{proof}

In essence, Theorem \ref{thm:main} shows that the deterministic equivalent $\bar \mQ_{\tilde\mX}(z)$ is composed of two main terms: a diagonal matrix $\zeta\cdot(1+\delta_1)\mI_n$, which describes the behaviour of the noise in the data model (both adjacency and node features), and an informative rank-2 matrix $\mU \left[ \mB^{-1} + \zeta \mT \right]^{-1}\mU^\intercal$ which correlates with the vector of labels $\bar \vy$ if the adjacency matrix and/or the node features are informative, i.e., values $\gamma_g$ and $\gamma_f,$ respectively, are sufficiently large. Figure \ref{fig:abc}(a) and (b) depict a histogram of the eigenvalues of $\tilde\mX\tilde\mX^\intercal$ which converges to the limiting distribution described by Theorem \ref{thm:main}, as well as its dominant eigenvector which correlates with $\bar \vy.$ Importantly, our analysis allows us to conclude that when the graph structure is completely noisy (i.e., $\eta=0$), the dominant eigenvector of $\tilde\mX\tilde\mX^\intercal$ is no longer aligned with $\bar \vy$ even if the node features are informative (i.e., $\gamma_f$ large) as will be clarified 
in Corollary \ref{cor:noisy}. 
\begin{corollary}[Case $\eta=0$]\label{cor:noisy} Recall the notation and Assumptions of Theorem \ref{thm:main}, for $\eta=0$ (i.e., a non-informative graph structure), $\bar\mQ_{\tilde \mX}(z)$ takes the form
\begin{align}
    \bar \mQ_{\tilde \mX}(z) = \zeta \cdot (1+\delta_1) \left( \mI_n - \frac{\zeta^2 \gamma_f}{c + \zeta \nu \gamma_f} \vphi \vphi^\intercal \right). \label{eqn:cor}
\end{align}
And, for $\hat \vy$ the eigenvector of $\tilde\mX\tilde\mX^\intercal$ corresponding to its largest eigenvalue, $\vert \bar \vy^\intercal \hat \vy \vert^2 \to_{n\to \infty} 0$.
\end{corollary}
\begin{proof}[Sketch of proof] 
Expression \eqnref{eqn:cor} follows from Theorem \ref{thm:main} 
by simply taking the limit as $\eta\to 0$. The second part of the Corollary is obtained by computing $\vert \bar \vy^\intercal \hat \vy \vert^2 = \frac{-1}{2i\pi}\oint_\Gamma \bar \vy^\intercal \mQ_{\tilde \mX}(-z) \bar \vy dz$ where $\Gamma$ is a small positively oriented complex contour surrounding the largest eigenvalue of $\tilde\mX\tilde\mX^\intercal$. Hence, using $\bar \mQ_{\tilde \mX}(z)$ as a proxy allows us to state $\vert \bar \vy^\intercal \hat \vy \vert^2 + \frac{1}{2i\pi}\oint_\Gamma \bar \vy^\intercal \bar\mQ_{\tilde \mX}(-z) \bar \vy dz \to 0 $ almost surely as $n\to \infty$. The final result is obtained by showing that $\bar \vy^\intercal \vphi \vphi^\intercal \bar \vy$ concentrates around its expectation with $\mathbb E \left[ \bar \vy^\intercal \vphi \vphi^\intercal \bar \vy\right] = \frac1n\Var[\bar\vy^\intercal\mN\bar \vy] = \frac{\nu}{n}\to 0$ and by evaluating $\frac{1}{2i\pi}\oint_\Gamma \zeta(-z)(1+\delta_1(-z))dz=0$.
\end{proof}
The main conclusion from Corollary \ref{cor:noisy} is that when the graph structure is completely random (when $\eta=0$, $\Vert \tilde \mA - \frac{1}{\sqrt n}\mN \Vert \to 0$), the largest eigenvector of $\tilde\mX\tilde\mX^\intercal$ (which is intuitively supposed to be informative) does not correlate with the labels vector $\bar \vy$ independently of the information contained in the node features. 
To overcome this issue, 
we propose in the now following Section \ref{sec:kernel-theory} to utilise node feature kernels to ensure the preservation of node feature information.

\subsection{Message Passing through Node Feature Kernels}\label{sec:kernel-theory}
As discussed in Section \ref{subsec:yyt}, when the graph structure (in the extreme case) is completely random, the  \textit{random} GCN model fails to extract information from the node features. To make the message passing informative and thereby to robustify the GCN, we propose to consider the operator $\tilde \mA + \bar \mK$ instead of $\tilde \mA,$ where $\bar \mK$ is a kernel matrix computed on the node features $\mX$. Indeed, let $\mK$ be a matrix with entries $K_{ij}=\kappa(\vx_i^\intercal \vx_j)$ for some smooth function $\kappa:\sR\to \sR$. Relying on \citep{el2010spectrum}, the kernel matrix $\mK$ can be approximated in spectral norm asymptotically as $n\to \infty$ by
\begin{align}
     \tilde \mK =\kappa(0) \vone_n\vone_n^\intercal + \kappa'(0) \left( \frac{\gamma_f}{c}\bar \vy\bar \vy^\intercal + \mZ \mZ^\intercal \right) + \mDelta,
\end{align}
where $\mDelta = \frac{\kappa''(0)}{2p}\vone_n\vone_n^\intercal + (\kappa(1) - \kappa(0) - \frac{\gamma_f}{c}\kappa'(0))\mI_n$. Hence, considering the matrix $\tilde \mA + \mP \mK \mP$, where $\mP = \mI_n - \frac1n \vone_n\vone_n^\intercal$ (the centring matrix), maintains the informative nature of the message passing step (through the term $\frac{\gamma_f}{c}\bar \vy \bar \vy^\intercal$) even in the case where the operator $\tilde \mA$ is not informative. 
Intuitively, the addition of the node feature kernel can be interpreted as considering both the originally recorded graph and a node feature similarity graph in the message passing architecture. This addition gives GNNs the necessary expressive ability to preserve information present in the node features, which is lost in the case of uninformative or noisy graph structures. 

Figure \ref{fig:abc}(c) shows the performance of different message passing strategies (compared to random MLP and spectral clustering; involving only node features or adjacency matrix, respectively) which confirms the effectiveness of introducing a node feature kernel in the regime where the graph similarity is noisy (i.e., low values of $\eta$), a property which is also validated for practical GCNs as we will discuss in Section \ref{sec:experiments}. 


\section{EXPERIMENTS} \label{sec:experiments}
In order to validate our theoretical findings in real-world scenarios, we experiment on the node classification task using GCNs on perturbed data. 
In Section~\ref{sec:result}, we begin by justifying the use of the \emph{random} GCN in the theoretical analysis. 
Then, we discuss results from experiments on both synthetic and real-world graphs involving a perturbation scheme on their edges. 
We show that the observed phenomena extend to deeper GCN architectures, and cases with node feature perturbations. Finally, in the case of deeper GCN models under structural perturbation, we demonstrate that our proposed method is comparable with state-of-the-art GNN models also placing a particular emphasis on the  node features and can be further improved when combined with other techniques.  

We work on synthetic SBM graphs aligned with Assumption~\ref{ass:conditions}, with the intra-community link probability being $q^2(1+\frac{\eta}{\sqrt{n}})$ and the inter-community link probability being $q^2$. We vary the parameter $\eta$ to generate SBM graphs with different types of community structure and keep other parameters fixed as $q=0.5,$ $p=2500,$ $n=1600,$ $\mu=(2, 0, \ldots, 0)$. 

We furthermore work with six real-world datasets which often serve as node classification benchmarks. 
These are the three well-studied citation networks of Cora, CiteSeer and PubMed \citep{SenNamataBilgicGetoorGalligherEliassi-Rad2008}, an extended version of Cora \citep{bojchevski2018deep}, called CoraFull, an Amazon co-purchase graph of Photo and a Co-author network from the authors of Computer Science (CS) \citep{shchur2018pitfalls}.
We follow the semi-supervised node classification setting proposed by \citet{pmlr-v48-yanga16}, i.e., we use their train/valid/test split for Cora/CiteSeer/PubMed, and we randomly sample $20$ nodes from each class as training set, $500$ nodes as validation set and another $1000$ nodes as test set for CoraFull/Photo/CS. Each experiment is repeated 10 times. Implementation details and a summary of dataset statistics can be found in the Supplementary Material Section \ref{appenx:datasets}. 

\paragraph{Perturbation Scheme}
The studied perturbation scheme involves both \emph{edge-deletion} noise, where a certain amount of existing edges are randomly sampled and removed from original graph,  
and \emph{edge-insertion} noise, where we add a certain amount of connections sampled from the non-existing edges in the original graph. We consider scenarios where edges are removed or added or both. The ratio of edges changed, i.e., the perturbation ratio, is denoted by $\alpha$ for \emph{edge-deletion} and $\beta$ for \emph{edge-insertion}.
A node feature kernel matrix is added to study its impact in practice, 
as shown in the following equation,\begin{equation}\label{eq:kernel}
    \mX^{(i+1)} = \sigma\left( (\epsilon\hat{\mA} + (1-\epsilon)\mathcal{N}(\mK) )\mX^{(i)}\mW \right),
\end{equation}
where $\hat{\mA}$ is the GCN message passing operator of the perturbed graph, $\mathcal{N}(\mK) = \mD_{\mK}^{-1/2} \mK \mD_{\mK}^{-1/2}$ is the normalised kernel matrix built from node features ($\mD_{\mK} = \mathrm{diag}(\mK \vone_n)$). We are degree normalising the kernel to match the graph representation of the GCN message passing operator.

\paragraph{Node Feature Kernels} As stated in Section~\ref{sec:kernel-theory}, we use the kernel to introduce information from the node features to the message passing structure. The choices of qualified smooth kernel functions are many. In our experiment, we perform a proof of concept using the simple linear kernel, defined as the inner product between node features $K_{ij} = \vx_i^\intercal \vx_j.$
\begin{figure}[tp!]
\centering
\includegraphics[width=\linewidth]{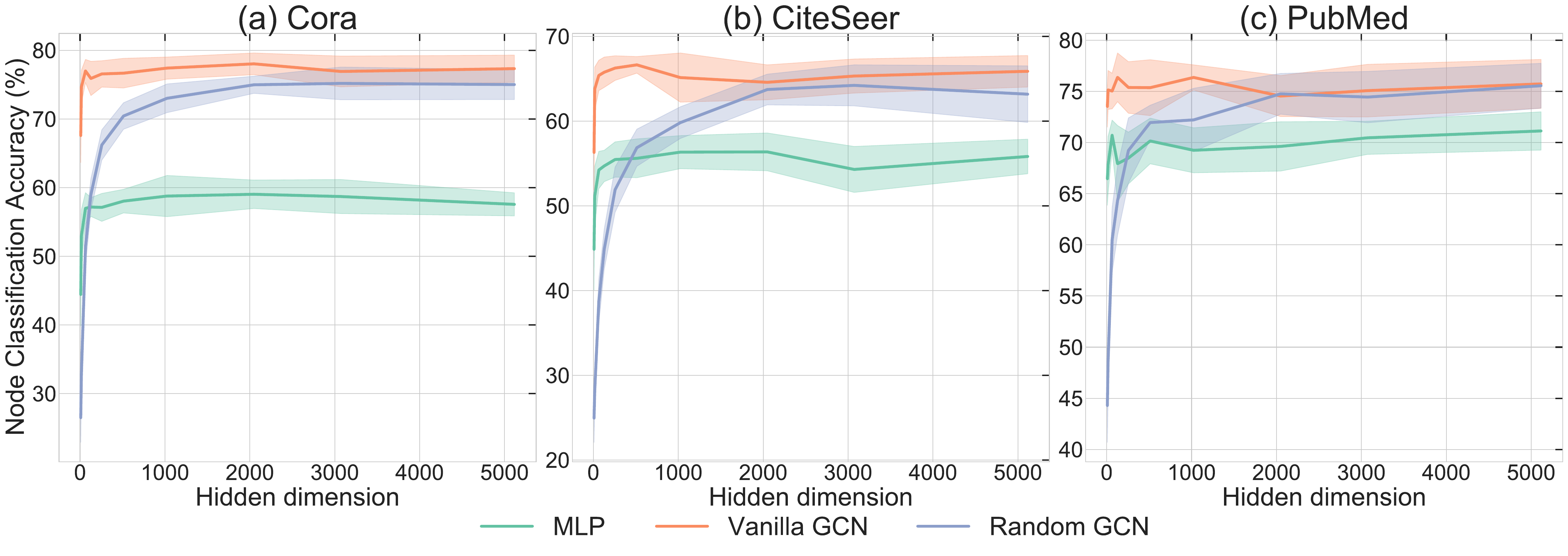}
\caption{Performance change over the embedding dimension with different models: \emph{random} GCN, vanilla GCN and MLP.}\label{fig:random-feature}
\end{figure}

\paragraph{Kernel Sparsification}
Using the full kernel matrix, where the kernel value is recorded between every pair of nodes, is both computationally costly and may incorporate redundant information. Therefore, we adopt a sparsification method using the adjacency matrix of the graph, with which \eqnref{eq:kernel} can be rewritten as,
\begin{equation}\label{eq:ksparse}
    \!\!\!\!\mX^{(i+1)}\! =\! \sigma\left( (\epsilon\hat{\mA} + (1-\epsilon) \mathcal{N}(\mK\circ\hat{\mA}))\mX^{(i)}\mW \right),
\end{equation}
where $\circ$ denotes Hadamard product.
A consequence of this sparsification method is that the added computational cost stemming from the consideration of the node feature kernel is linear in the number of edges in the graph $|E|,$ where $E$ denotes the graph's edge set, and the node feature dimension $p,$ i.e., of order $\mathcal{O}(|E|p).$ Our initial experiments sparsifying the kernel matrix by using a threshold below which all entries are set to zero resulted in worse performance and introduced the threshold as an extra hyperparameter. Therefore, we chose to only pursue the sparsification scheme in \eqnref{eq:ksparse}. This preliminary  observation could be a result of the particular node features that are recorded in our datasets. 
\begin{table*}[!t]
\caption[Performance of GCN with and without node-feature kernel under perturbation on synthetic SBM graphs. Results are set to \textbf{bold} if they are significantly better than their counterparts.]{Performance of GCN with node-feature kernel under perturbation on synthetic SBM graphs. The best results are set to \textbf{bold} if their range of one standard deviation does not overlap with the standard deviation of their counterpart. } \label{tab:sbm} 
\def\arraystretch{1.0}
\scriptsize
\begin{center}
\begin{tabular}{cc|cc|cc|cc|}
& & \multicolumn{2}{c|}{SBM($q=0.5$, $\eta=0$)} & \multicolumn{2}{c|}{SBM($q=0.5$, $\eta=4$)} & \multicolumn{2}{c|}{SBM($q=0.5$, $\eta=-4$)}  \\
& ($\alpha$, $\beta$) & GCN & GCN-k & GCN & GCN-k & GCN & GCN-k \\
\hline
& (0.0, 0.0) & $50.53 \pm 0.49$ & $\mathbf{66.36 \pm 0.81}$ & $\mathbf{64.42 \pm 0.43}$ & $62.26 \pm 1.04$ & $\mathbf{63.20 \pm 0.94}$ & $61.03 \pm 1.08$\\

\multirow{2}{*}{Deletion} & (0.2, 0.0) & $51.03 \pm 0.56$ & $\mathbf{65.44 \pm 1.07}$ & $58.63 \pm 0.68$ & $\mathbf{71.57 \pm 1.42}$ & $\mathbf{60.89 \pm 0.83}$ & $54.91 \pm 1.00$\\

& (0.5, 0.0) & $49.29 \pm 0.59$ & $\mathbf{64.14 \pm 1.01}$ & $60.76 \pm 1.29$ & $\mathbf{68.80 \pm 2.04}$ & $58.41 \pm 1.11$ & $59.51 \pm 2.47$\\ 

\multirow{2}{*}{Insertion} & (0.0, 0.5) & $50.57 \pm 0.75$ & $\mathbf{68.57 \pm 1.25}$ & $60.49 \pm 0.40$ & $\mathbf{68.20 \pm 1.38}$ & $58.82 \pm 1.16$ & $\mathbf{63.54 \pm 0.97}$\\

& (0.0, 1.0) & $49.19 \pm 0.47$ & $\mathbf{59.31 \pm 0.58}$ & $53.67 \pm 1.11$ & $\mathbf{66.57 \pm 1.73}$ & $54.87 \pm 0.53$ & $\mathbf{60.84 \pm 0.75}$\\ 

\multirow{2}{*}{Delet.+Insert.} & (0.5, 0.5) & $49.26 \pm 0.59$ & $\mathbf{68.84 \pm 0.86}$ & $50.50 \pm 0.37$ & $\mathbf{63.36 \pm 1.67}$ & $50.94 \pm 0.86$ & $\mathbf{63.02 \pm 0.91}$\\

& (0.5, 1.0) & $49.84 \pm 0.69$ & $\mathbf{65.49 \pm 1.22}$ & $48.34 \pm 0.22$ & $\mathbf{60.16 \pm 1.21}$ & $49.23 \pm 0.45$ & $\mathbf{59.64 \pm 1.33}$ \\
\end{tabular}
\end{center}
\end{table*}
\begin{table*}[t]
\caption{Performance of GCN with and without node-feature kernel under perturbation on six real-world datasets. The format follows Table \ref{tab:sbm}.} \label{tab:real}
\def\arraystretch{1.0}
\scriptsize
\begin{center}
\begin{tabular}{cc|cc|cc|cc|}
& & \multicolumn{2}{c|}{Cora} & \multicolumn{2}{c|}{CiteSeer} & \multicolumn{2}{c|}{PubMed}  \\
& ($\alpha$, $\beta$) & GCN & GCN-k & GCN & GCN-k & GCN & GCN-k \\
\hline
& (0.0, 0.0) & $\mathbf{79.37 \pm 0.65}$ & $76.94 \pm 0.35$ & $67.45 \pm 0.82$ & $68.08 \pm 0.91$ & $76.04 \pm 0.67$ & $74.68 \pm 0.76$ \\

\multirow{2}{*}{Deletion} & (0.2, 0.0) & $76.15 \pm 0.81$ & $74.83 \pm 1.24$ & $66.79 \pm 0.57$ & $66.94 \pm 0.82$ & $\mathbf{75.82 \pm 0.99}$ & $74.28 \pm 0.39$ \\

& (0.5, 0.0) & $72.49 \pm 0.50$ & $71.18 \pm 1.00$ & $63.53 \pm 0.75$ & $64.84 \pm 1.14$
& $73.95 \pm 0.64$ & $73.25 \pm 0.75$ \\ 

\multirow{2}{*}{Insertion} & (0.0, 0.5) & $68.57 \pm 0.73$ & $\mathbf{73.10 \pm 1.10}$ & $59.85 \pm 0.89$ & $\mathbf{66.11 \pm 1.34}$ & $64.18 \pm 0.67$ & $\mathbf{72.38 \pm 0.79}$ \\

& (0.0, 1.0) & $64.14 \pm 1.02$ & $\mathbf{73.36 \pm 0.98}$ & $55.39 \pm 0.93$ & $\mathbf{64.94 \pm 0.77}$ & $60.56 \pm 0.80$ & $\mathbf{71.31 \pm 0.51}$ \\ 

\multirow{2}{*}{Delet.+Insert.} & (0.5, 0.5) & $54.98 \pm 1.13$ & $\mathbf{66.46 \pm 1.03}$ & $52.84 \pm 0.68$ & $\mathbf{59.03 \pm 1.04}$ & $62.62 \pm 0.72$ & $\mathbf{70.32 \pm 0.82}$ \\

& (0.5, 1.0) & $48.09 \pm 0.88$ & $\mathbf{62.52 \pm 0.59}$ & $42.28 \pm 1.07$ & $\mathbf{58.07 \pm 1.34}$ & $53.25 \pm 1.57$ & $\mathbf{69.65 \pm 0.60}$  \\
\end{tabular}\\
 \vspace{0.05cm}
\begin{tabular}{cc|cc|cc|cc|}
& & \multicolumn{2}{c}{CoraFull} & \multicolumn{2}{c|}{Photo} & \multicolumn{2}{c|}{CS} \\
& ($\alpha$, $\beta$) & GCN & GCN-k & GCN & GCN-k & GCN & GCN-k \\
\hline
& (0.0, 0.0) & $57.21 \pm 0.84$ & $56.88 \pm 0.48$ & $90.94 \pm 0.49$ & $90.09 \pm 0.65$ & $92.89 \pm 0.41$ & $92.63 \pm 0.31$ \\

\multirow{2}{*}{Deletion} & (0.2, 0.0) & $\mathbf{57.25 \pm 0.67}$ & $55.56 \pm 0.69$ & $91.87 \pm 0.40$ & $92.19 \pm 0.45$ & $90.58 \pm 0.48$ & $90.89 \pm 0.48$ \\
& (0.5, 0.0) & $53.90 \pm 0.70$ & $54.62 \pm 0.87$ & $\mathbf{91.10 \pm 0.40}$ & $87.97 \pm 0.54$ & $89.75 \pm 0.60$ & $\mathbf{91.27 \pm 0.67}$ \\

\multirow{2}{*}{Insertion} & (0.0, 0.5) & $48.11 \pm 0.89$ & $\mathbf{51.79 \pm 0.65}$ & $82.79 \pm 1.43$ & $84.18 \pm 1.27$ & $87.16 \pm 0.65$ & $\mathbf{90.81 \pm 0.70}$ \\
& (0.0, 1.0) & $41.76 \pm 1.03$ & $\mathbf{51.91 \pm 1.00}$ & $72.70 \pm 6.40$ & $\mathbf{79.58 \pm 1.80}$ & $80.34 \pm 0.80$ & $\mathbf{90.61 \pm 0.37}$ \\

\multirow{2}{*}{Delet.+Insert.} & (0.5, 0.5) & $34.70 \pm 0.47$ & $\mathbf{46.50 \pm 0.61}$ & $69.70 \pm 3.70$ & $\mathbf{74.65 \pm 2.36}$ & $73.75 \pm 0.98$ & $\mathbf{87.28 \pm 0.72}$ \\
& (0.5, 1.0) & $27.50 \pm 1.04$ & $\mathbf{43.04 \pm 0.77}$ & $61.13 \pm 2.49$ & $63.73 \pm 5.04$ & $66.26 \pm 0.95$ & $\mathbf{87.51 \pm 0.58}$ \\
\end{tabular}
\end{center}
\end{table*}

\subsection{Experiment Analysis} \label{sec:result}


\paragraph{Asymptotic Analysis of Random GCN} 
To validate the practical applicability of the theoretical analysis in Section \ref{sec:theory},
we study the asymptotic behaviour of the \emph{random} GCN, the vanilla GCN and a MLP baseline when the hidden dimension of node features grows. In Figure \ref{fig:random-feature}, we observe that with increasing hidden dimension, the performance of both the vanilla GCN and MLP remains stable, while the performance of \emph{random} GCN converges to vanilla GCN's accuracy. Between hidden dimensions of $2000$ and $3000$ the performance of \emph{random} GCN starts to match that of vanilla GCN. 
Hence, we have given an empirical indication of the conditions under which our theoretical model, the \emph{random} GCN, and the vanilla GCN are equivalent. 

\paragraph{Robustness to Structural Noise: Synthetic SBM}
We first test the performance of the proposed model on synthetic SBM graphs. Three types of SBM graph are considered, which are distinguished by the parameter $\eta$. $\eta=0$, $\eta=4$ and $\eta=-4$ correspond to the cases where the synthetic graph has no, a homophilic and a heterophilic community structure, respectively. We experiment on perturbation scenarios with different \emph{egde-deletion} and \emph{edge-insertion} ratios. For each scenario, we record the performance of the vanilla GCN as well as its performance after adding the node-feature kernel, denoted by an appendage ``-k''\footnote{If not specified, the weight coefficient $\epsilon$ of the added kernel in graph propagation is set to $0.5$.} in Table \ref{tab:sbm}. 

When a graph has no community structure, structural perturbation has little impact and adding node-feature information boosts the performance. When the graphs are homophilic and with a clear community structure, the impact from graph-structural noise become more visible. Adding a node-feature kernel significantly improves the model's robustness against \emph{edge-deletion} and \emph{edge-insertion} noise and their mix. The same conclusion can be drawn on heterophilic graphs perturbed by \emph{edge-insertion} noise. 

\paragraph{Robustness to Structural Noise: Real-World datasets}
In this set of experiments, we study the change of model performance 
under the same perturbation setting on real-world datasets. Table~\ref{tab:real} shows the results over different perturbation scenarios.

Naturally the performance decreases gradually when edges are removed or added, as we can see from each column. \emph{Edge-insertion} noise seems to have a larger impact on the performance than the \emph{edge-deletion} noise on these real-world datasets. But this influence can be largely compensated by adding the node feature kernel. Unlike results from synthetic SBM graphs, for \emph{edge-deletion} noise, the addition of the kernel on real-world datasets seems to have almost no impact.

\begin{table*}[t]
\caption{Performance of GCN with and without node-feature kernel under perturbation on deep GCN models, compared with jump knowledge and GCNII. The format follows Table \ref{tab:sbm}, where in addition we underline the second best result.} \label{tab:multi-layer}
\def\arraystretch{1.0}
\scriptsize
\begin{center}
\begin{tabular}{cc|c|c|c|c|c|c|}
& & \multicolumn{6}{c}{Cora} \\
& ($\alpha$, $\beta$) & GCN & GCN-k ($\epsilon=0.5$) & GCN-k ($\epsilon=0.2$) & GCN-jk & GCNII & GCN-k-jk \\
\hline
 & (0.0, 0.0) & $79.05 \pm 1.36$ & $79.67 \pm 1.17$ & $79.27 \pm 1.50$ & $80.39 \pm 1.13$ & $79.62 \pm 1.24$ & $79.74 \pm 0.75$ \\


Deletion& (0.5, 0.0) & $74.20 \pm 1.15$ & $72.22 \pm 1.25$ & $72.74 \pm 1.23$ & $74.97 \pm 0.71$ & $73.87 \pm 0.91$ & $74.16 \pm 0.77$
\\


Insertion& (0.0, 1.0) & $53.48 \pm 3.49$ & $63.20 \pm 1.78$ & $\underline{71.24 \pm 1.04}$ & $65.86 \pm 0.62$ & $68.26 \pm 1.15$ & $\mathbf{72.94 \pm 0.67}$
\\

Delet.+Insert.& (0.5, 0.5) & $47.40 \pm 1.73$ & $57.34 \pm 1.53$ & $60.77 \pm 1.57$ & $57.94 \pm 0.92$ & $\underline{62.54 \pm 1.55}$ & $\mathbf{67.11 \pm 0.93}$
\\

\end{tabular}
 \vspace{0.05cm}
\begin{tabular}{cc|cc|cc|cc|}
 & &  \multicolumn{6}{c}{CS} \\
& ($\alpha$, $\beta$) & GCN & GCN-k ($\epsilon=0.5$) & GCN-k ($\epsilon=0.2$) & GCN-jk & GCNII & GCN-k-jk \\
\hline
& (0.0, 0.0) & $88.44 \pm 0.84$ & $89.81 \pm 0.52$ & $91.64 \pm 0.39$ & $90.19 \pm 0.59$ & $\mathbf{92.13 \pm 0.39}$ & $\underline{91.73 \pm 0.26}$ \\


Deletion& (0.5, 0.0) & $86.68 \pm 0.57$ & $86.17 \pm 1.06$ & $88.91 \pm 0.62$ & $88.44 \pm 0.69$ & $\underline{90.01 \pm 0.69}$ & $\mathbf{91.43 \pm 0.60}$ \\


Insertion& (0.0, 1.0) & $35.84 \pm 6.91$ & $81.06 \pm 3.94$ & $88.27 \pm 0.92$ & $81.70 \pm 0.63$ & $\underline{89.33 \pm 1.02}$ & $\mathbf{91.42 \pm 0.48}$ \\

Delet.+Insert.& (0.5, 0.5) & $45.08 \pm 4.82$ & $76.27 \pm 1.08$ & $82.23 \pm 1.08$ & $73.08 \pm 1.07$ & $\mathbf{88.66 \pm 0.70}$ & $\underline{87.53 \pm 0.85}$ \\

\end{tabular}
\end{center}
\end{table*}
\begin{table*}[t!]
\caption{Performance of GCN with and without node-feature kernel under both graph-structural and node-feature perturbation on PubMed dataset. The format follows Table \ref{tab:sbm}.} \label{tab:node-feat-noise}
\def\arraystretch{1.0}
\scriptsize
\begin{center}
\begin{tabular}{cc|c|c|c|c|c|c|}
& & \multicolumn{2}{c|}{$\gamma=0.5$} & \multicolumn{2}{c|}{$\gamma=1.0$} & \multicolumn{2}{c|}{$\gamma=5.0$} \\
& ($\alpha$, $\beta$) & GCN & GCN-k & GCN & GCN-k & GCN & GCN-k\\
\hline
 & (0.0, 0.0) & $72.78 \pm 1.62$	& $73.88 \pm 2.04$	
 & $68.40 \pm 2.95$ & $66.82 \pm 1.81$ & $42.43 \pm 2.95$ & $39.14 \pm 2.6$	
 \\


Deletion& (0.5, 0.0) & $71.61 \pm 1.31$ & $70.33 \pm 2.38$	
& $63.81 \pm 3.43$ & $64.93 \pm 2.14$ & $40.41 \pm 1.92$ & $38.11 \pm 1.99$	
\\


Insertion& (0.0, 1.0) &  $58.23 \pm 1.79$ & $\mathbf{68.14 \pm 1.67}$	
& $55.37 \pm 3.33$ & $\mathbf{63.35 \pm 2.30}$	& $38.25 \pm 1.93$	& $38.72 \pm 1.93$	
\\
\end{tabular}
\end{center}
\end{table*}

\paragraph{Deeper GCN architecture and Benchmark Models} The previous experiments are based on a single-layer GCN model. In practice, the best-performing GCN models on these datasets often contain several message-passing layers and therefore, we want to observe whether our theoretical results can be extrapolated to the multi-layer case. We build a $4$-layer GCN model
and repeat our experiments on the Cora and CS datasets.
The results are shown in Table~\ref{tab:multi-layer}. Results of the remaining datasets can be found in the Supplementary Material Section \ref{appenx:additionaL_experiments}.
Moreover, we also consider varying the weight coefficient of the added node-feature kernel in graph propagation and compare with two models specifically proposed for deep GNN architectures: Jumping Knowledge (JK, \citet{xu2018representation}) and a GCN with residual connections and an identity mapping (GCNII, \citet{chen2020simple}). Although not designed to tackle the graph-structural perturbation problem that we study in this paper, the two models both utilise node feature information, which makes them reasonable baselines for our proposed method.

As we can see from Table~\ref{tab:multi-layer}, there is no fundamental change in the trends observed for the single-layer model. Adding the node-feature kernel helps robustify model performance against perturbation when edges are inserted and/or removed. The improvement is even more significant when the weight coefficient $\epsilon$ on the kernel in graph propagation is increased. Additionally, our proposed method, with a high weight on the kernel, is comparable to GCNII model and in general performs better than the JK model. However, since JK can easily be combined with our method, the node-feature kernel and JK architecture yields the best-performing model as can be seen in the rightmost column of Table~\ref{tab:multi-layer}.       

\paragraph{Node feature noise} 
We now study how node feature noise interacts with our proposed model. 
We perform the experiments with the same perturbation setting on the PubMed dataset.
But add Gaussian noise $\mathcal{N}(\vzero, \gamma ~\mathrm{diag}(\sigma_i))$ to the node features where $\sigma_i$ is the estimated variance of the $i^{\mathrm{th}}$ feature and $\gamma$ is a scaling parameter. 
The results are recorded in Table~\ref{tab:node-feat-noise}. We observe that only when the node feature noise is five times larger than 
the graph structure noise ($\gamma=5$), the addition of the node feature kernel stops to benefit the model performance. 
Our previous findings still hold if the node feature noise is reasonably small ($\gamma\leq1$).

\section{CONCLUSION}

We have introduced 
the \emph{random} GCN, which we analysed theoretically using random matrix theory. Our analysis allowed 
us to conclude that \textit{perturbations of the graph structure strongly influence the performance of the GCN regardless of the information contained in the node features.} 
For stochastic blockmodel graphs the presence of community structure (and the degree to which this structure is present) is required (beneficial) for a message passing scheme which leads to eigenvectors of the message passing operator's Gram matrix that align with the node labels. 
These conclusions were confirmed in multiple experiments with the standard GCN architecture on synthetic and real-world datasets.
On both synthetic and real-world data we observe \textit{the introduction of a node feature kernel to the GCN's message passing scheme to significantly improve the performance of the GCN in the presence of a noisy graph structure}. 

\subsubsection*{Acknowledgements}
The work of Dr. Johannes Lutzeyer and Prof. Michalis Vazirgiannis is supported by the ANR chair AML-HELAS (ANR-CHIA-0020-01).



\bibliographystyle{agsm}
\bibliography{references}





\clearpage
\appendix

\thispagestyle{empty}

\onecolumn \makesupplementtitle

\section{RANDOM MATRIX THEORY BACKGROUND}\label{appenx:theory}
We begin by recalling several random matrix theory tools that are needed to establish our main results. First, we recall a fundamental result (Theorem \ref{thm:sample}) from \cite{louart2018concentration} which provides a deterministic equivalent for the \textit{resolvent} of a \textit{sample covariance} matrix.  
\begin{theorem}[Deterministic equivalent for sample covariance matrices \cite{louart2018concentration}] \label{thm:sample} Let $\mM \in \sR^{n\times n}$ such that $\Vert \mM \mM^\intercal \Vert < \infty$\footnote{$\Vert \mM \mM^\intercal \Vert$ remains constant as $n$ goes to infinity.} w.r.t. $n$ and $\tilde\mZ\in \sR^{n\times p}$ some random matrix with i.i.d. entries having zero mean, unit variance and a finite forth order moment. In the limit $n\to \infty$ with $p/n\to c\in (0, \infty)$, the resolvent $\mQ(z) = \left( \frac1p \mM \tilde\mZ \tilde\mZ^\intercal \mM^\intercal + z \mI_n \right)^{-1}$ for $z\in \sC$ with $\Im(z)>0$, admits a deterministic equivalent $\bar \mQ(z)$ defined as
\begin{align*}
    \bar \mQ(z) = \left( \frac{\mM\mM^\intercal}{1+\delta(z)} + z \mI_n \right)^{-1},
\end{align*}
where $\delta(z)$ is the unique solution to the fixed point equation $\delta(z) = \frac1p \Tr\left( \mM^\intercal \bar \mQ(z) \mM \right)$.
\end{theorem}
\begin{proof}
Denote $\va_i = \mM \tilde\vz_i $, hence
\begin{align*}
    \mQ(z) = \left( \frac1p \sum_{i=1}^n \va_i \va_i^\intercal + z\mI_n \right)^{-1} = \mQ_{-i} - \frac{\mQ_{-i} \frac1p \va_i \va_i^\intercal \mQ_{-i} }{1 + \frac1p \va_i^\intercal \mQ_{-i} \va_i},
\end{align*}
where $\mQ_{-i} = \left( \frac1p \sum_{j\neq i}^n \va_j \va_j^\intercal + z\mI_n \right)^{-1}$, and we also have
\begin{align*}
    \mQ(z) \va_i = \frac{\mQ_{-i}\va_i}{1 + \frac1p \va_i^\intercal \mQ_{-i} \va_i}.
\end{align*}
A deterministic equivalent for $\mQ(z)$ (which approximates $\mathbb E [\mQ(z)]$) is of the form $\bar \mQ(z) = \left( \mF + z \mI_n \right)^{-1}$ for some deterministic matrix $\mF$, by computing the difference $\bar\mQ(z) - \mathbb E[\mQ(z)]$ using the above identities, we obtain
\begin{align*}
    \bar \mQ(z) - \mathbb E[ \mQ(z) ] &= \frac1n \sum_{i=1}^n \mathbb E \left[ \mQ_{-i} \left( \frac{\va_i \va_i^\intercal}{1 + \frac1p \va_i^\intercal \mQ_{-i} \va_i} - \mF \right) \bar\mQ(z) \right]\\ 
    &+ \frac{1}{n^2} \sum_{i=1}^n \mathbb E \left[ \mQ_{-i}(z) \va_i \va_i^\intercal \mQ_{-i}(z) \mF \bar \mQ(z) \right].
\end{align*}
It can be shown that the matrix $\frac{1}{n^2} \sum_{i=1}^n \mathbb E \left[ \mQ_{-i}(z) \va_i \va_i^\intercal \mQ_{-i}(z) \mF \bar \mQ(z) \right]$ has a vanishing operator norm as $n\to \infty$. Therefore, $\mF$ can be taken as
\begin{align*}
    \mF = \frac{   \mathbb E \left[ \va_i \va_i^\intercal \right]}{1 + \mathbb E \left[  \frac1p \va_i^\intercal \mQ_{-i} \va_i\right]} = \frac{   \mM \mM^\intercal }{1 + \frac1p \Tr ( \mM^\intercal \mathbb E \left[  \mQ_{-i}\right] \mM)}
\end{align*}
Which provides the defined deterministic equivalent in Theorem \ref{thm:sample}.
\end{proof}
Another useful result which is known as the perturbation lemma \citep{silverstein1995empirical} is also needed here.
\begin{lemma}[Perturbation lemma \cite{silverstein1995empirical}]\label{lem:perturbation} Let $\mA, \mB\in \sR^{n\times n}$ some symmetric matrices, $\vu\in \sR^n$, $\gamma\in \sR$ and $z\in \sC$ with $\Im(z)>0$, then
\begin{align*}
    \left\vert \Tr \left( \mA (\mB + \gamma \vu \vu^\intercal + z \mI_n)^{-1} \right) - \Tr \left( \mA (\mB  + z \mI_n)^{-1} \right) \right\vert \leq \frac{\Vert \mA \Vert}{\vert \Im(z) \vert}.
\end{align*}
In particular, for $\mA = \frac1n \mI_n$, we have $\frac1n \Tr(\mB + \gamma \vu \vu^\intercal + z \mI_n)^{-1} = \frac1n \Tr(\mB +  z \mI_n)^{-1} + \mathcal{O}(n^{-1})$, which shows that the spectral measure of $\mB + \gamma \vu \vu^\intercal$ is asymptotically close to that of $\mB$ in the large $n$ limit.
\end{lemma}
Finally, we will need the Woodbury matrix identity from the following Lemma.
\begin{lemma}[Woodbury identity]\label{lem:woodbury}
Let $\mA\in \sR^{n\times n}$ and $\mB\in \sR^{k\times k}$ invertible and $\mU\in \sR^{n\times k}$, then
\begin{align*}
    \left( \mA + \mU \mB \mU^\intercal \right)^{-1} = \mA^{-1} - \mA^{-1} \mU \left( \mB^{-1} + \mU^\intercal \mA^{-1} \mU \right)^{-1} \mU^\intercal \mA^{-1}.
\end{align*}
\end{lemma}

\section{PROOF OF THEOREM \ref{thm:main}}\label{appenx:proof_thm}
The proof starts by establishing a random equivalent for the normalised Adjacency operator given by $\tilde \mA = \frac{1}{\sqrt n}\left( \mA - \vq \vq^\intercal \right)$. By Assumptions 2, 3 and 5, we have straightforwardly that, almost surely
\begin{align}
    \left \Vert \tilde \mA - \left( q^2 \eta \bar \vy \bar \vy^\intercal + \frac{1}{\sqrt n} \mN \right) \right \Vert \to 0,
\end{align}
where $\mN$ is a random matrix with i.i.d. entries of zero mean and variance $\nu = q^2(1-q^2)$.
Besides, since $\mathbb E \left[ \mX\mX^\intercal \right] = \frac{\Vert \vmu \Vert^2}{c}\bar \vy \bar \vy^\intercal + \mI_n$, letting $\gamma_f = \Vert \vmu \Vert^2$ and $\gamma_g = q^2 \eta$, we consider the equivalent multiplicative model for $\mY$ defined as
\begin{align*}
    \mY =  \left( \gamma_g \bar\vy \bar\vy^\intercal + \frac{1}{\sqrt n} \mN  \right) \left( \frac{\gamma_f}{c} \bar\vy \bar\vy^\intercal + \mI_n \right)^{\frac12} \mZ,
\end{align*}
where $\mZ$ is random matrix with i.i.d. entries of zero mean and variance $\frac1p$. Conditionally on $\mN$, applying Theorem \ref{thm:sample} for $\mM = \left( \gamma_g \bar\vy \bar\vy^\intercal + \frac{1}{\sqrt n} \mN  \right) \left( \frac{\gamma_f}{c} \bar\vy \bar\vy^\intercal + \mI_n \right)^{\frac12}$, a deterministic equivalent of $\mQ_\mY(z) = \left(\mY\mY^\intercal + z \mI_n \right)^{-1}$ is given by
\begin{align}
    \bar \mQ_{\mY \vert \mN}(z) &= \left( \frac{\left( \gamma_g \bar\vy \bar\vy^\intercal + \frac{1}{\sqrt n} \mN  \right) \left( \frac{\gamma_f}{c} \bar\vy \bar\vy^\intercal + \mI_n \right) \left( \gamma_g \bar\vy \bar\vy^\intercal + \frac{1}{\sqrt n} \mN  \right) }{1+\delta_1(z)} + z\mI_n  \right)^{-1}\\
    &= \left( \frac{\mU \mB \mU^\intercal + \frac1n \mN \mN^\intercal}{1+\delta_1(z)} + z\mI_n \right)^{-1},
\end{align}
where 
\begin{align}
    \mU = \left[ \bar \vy, \vphi \right] \in \sR^{n\times 2}, \quad \mB = \begin{bmatrix}
    \gamma_2^2 \left( \frac{\gamma_1}{c} + 1 \right) & \left( \frac{\gamma_1}{c} + 1 \right) \gamma_2 \\
    \left( \frac{\gamma_1}{c} + 1 \right) \gamma_2 & \frac{\gamma_1}{c}
    \end{bmatrix}
\end{align}
with $\vphi = \frac{1}{\sqrt n} \mN \bar \vy$, $\bar \vy = \vy / \sqrt n$ and $\delta_1(z) = \frac1p \Tr \left( \left( \mU \mB \mU^\intercal + \frac1n \mN \mN^\intercal \right) \bar \mQ_{\mY \vert \mN}(z) \right)$. Applying Lemma \ref{lem:perturbation}, $\delta_1(z)$ is simply the solution to $\delta_1(z) = \frac1p \Tr \left(  \frac1n \mN \mN^\intercal  \bar \mQ_{\mY \vert \mN}(z) \right)$. Moreover, defining the matrix $\bar \mQ_{\mY \vert \mN}^{-\mB}(z) = \left( \frac{ \frac1n \mN \mN^\intercal}{1+\delta_1(z)} + z\mI_n \right)^{-1}$, we have by Lemma \ref{lem:woodbury}
\begin{align}\label{eq:YbarN}
    \bar \mQ_{\mY \vert \mN}(z) = \bar \mQ_{\mY \vert \mN}^{-\mB}(z) - \bar \mQ_{\mY \vert \mN}^{-\mB}(z) \mU \left( (1+\delta_1(z))\mB^{-1} + \mU^\intercal \bar \mQ_{\mY \vert \mN}^{-\mB}(z) \mU  \right)^{-1} \mU^\intercal \bar \mQ_{\mY \vert \mN}^{-\mB}(z).
\end{align}
Again by Theorem \ref{thm:sample}, a deterministic equivalent of $\bar \mQ_{\mY \vert \mN}^{-\mB}(z)$ is given by
\begin{align}
    \bar \mQ_\mY^{-\mB}(z) = \left( \frac{\nu \mI_n}{(1+\delta_1(z))(1+\delta_2(z))}  + z \mI_n\right)^{-1} = \frac{(1+\delta_1(z))(1+\delta_2(z))}{\nu + z (1+\delta_1(z))(1+\delta_2(z))} \mI_n,
\end{align}
where $\delta_2(z)$ is the unique solution to $\delta_2(z) = \frac1n \Tr \left( \frac{\nu \mI_n}{(1+\delta_1(z))} \bar \mQ_\mY^{-\mB}(z) \right) = \frac{\nu (1+\delta_2(z))}{\nu + z (1+\delta_1(z))(1+\delta_2(z))}$, and similarly $\delta_1(z)$ satisfies $\delta_1(z) = \frac1c \frac{\nu (1+\delta_1(z))}{\nu + z (1+\delta_1(z))(1+\delta_2(z))}$. Therefore, replacing $\bar \mQ_{\mY \vert \mN}^{-\mB}(z)$ in (\ref{eq:YbarN}) by its deterministic equivalent $\bar \mQ_\mY^{-\mB}(z)$ and since $\mU^\intercal \mU \to \mT$ almost surely, provides the final result of Theorem 3.4.

\section{Proof of Corollary \ref{cor:noisy}}\label{appenx:proof_cor}
Following the same procedure as in Section \ref{appenx:proof_thm}, when $\eta=0$, a deterministic equivalent for $\mQ_\mY(z)$ takes the form
\begin{align}
    \bar \mQ_\mY(z) = \zeta(z) (1 + \delta_1(z)) \left( \mI_n - \frac{\zeta^2(z)\gamma_f}{c + \nu\gamma_f\zeta(z)} \vphi \vphi^\intercal \right).
\end{align}
By definition of the deterministic equivalent, we have almost surely
\begin{align}
    \vert \bar \vy^\intercal \hat \vy \vert^2 = \frac{-1}{2\pi i} \oint_\Gamma \bar \vy^\intercal \mQ_\mY(-z) \bar \vy dz \to_{n\to \infty} \frac{-1}{2\pi i} \oint_\Gamma \bar \vy^\intercal \bar \mQ_\mY(-z) \bar \vy dz.
\end{align}
Hence, we need to evaluate the Cauchy-integral $\frac{-1}{2\pi i} \oint_\Gamma \bar \vy^\intercal \bar \mQ_\mY(-z) \bar \vy dz$. In particular, the quadratic form $\bar \vy^\intercal \bar \mQ_\mY(z) \bar \vy$ evaluates as
\begin{align*}
    \bar \vy^\intercal \bar \mQ_\mY(z) \bar \vy = \zeta(z) (1 + \delta_1(z)) \left( 1 - \frac{\zeta^2(z)\gamma_f}{c + \nu\gamma_f\zeta(z)} \bar \vy^\intercal \vphi \vphi^\intercal \bar \vy \right) \to_{n\to \infty} \zeta(z) (1 + \delta_1(z)).
\end{align*}
Indeed, since the mapping $\mX\mapsto \frac{1}{\sqrt n} \bar \vy^\intercal \mX \bar \vy$ is $\frac{1}{\sqrt n}$-Lipschitz transformation w.r.t. the Frobenius norm $\Vert \cdot \Vert_F$, then we have the concentration inequality, for all $t\geq 0$
\begin{align}
    \mathbb P \left( \left\vert \frac{1}{\sqrt n} \bar \vy^\intercal \mN \bar \vy  - \mathbb E \left[ \frac{1}{\sqrt n} \bar \vy^\intercal \mN \bar \vy \right] \right\vert > t \right) \leq C\, e^{-\, (\sqrt n\, t / \nu)^2},
\end{align}
for some constant $C\geq 0$ independent of $n$. In particular, since $\mathbb E \left[ \frac{1}{\sqrt n} \bar \vy^\intercal \mN \bar \vy \right] = 0$, we have
\begin{align}
    \mathbb P \left( \left\vert \left(\frac{1}{\sqrt n} \bar \vy^\intercal \mN \bar \vy \right)^2 - \mathbb E \left[ \left( \frac{1}{\sqrt n} \bar \vy^\intercal \mN \bar \vy \right)^2 \right] \right\vert > t \right) \leq C\, e^{-\frac{\sqrt n\, t}{2\nu} },
\end{align}
which shows that $\bar \vy^\intercal \vphi \vphi^\intercal \bar \vy = \left(\frac{1}{\sqrt n} \bar \vy^\intercal \mN \bar \vy \right)^2$ concentrates around its mean value, with
\begin{align*}
    \mathbb E \left[ \left( \frac{1}{\sqrt n} \bar \vy^\intercal \mN \bar \vy \right)^2 \right] = \frac1n \Var \left[ \bar \vy^\intercal \mN \bar \vy \right] = \frac1n \sum_{i,j=1}^n \bar y_i^2 \bar y_j^2 \Var[N_{ij}] = \frac{\Vert \bar \vy \Vert^4\, \nu}{n} \to 0.
\end{align*}
Therefore, $\bar \vy^\intercal \vphi \vphi^\intercal \bar \vy\to 0$ almost surely as $n\to \infty$. The final step consists in evaluating the integral $\frac{1}{2\pi i} \oint_\Gamma \zeta(-z)(1+\delta_1(-z))dz = 0$ since the function $z\mapsto \zeta(-z)(1+\delta_1(-z))$ does not have singularities on the contour $\Gamma$. Indeed, this integral corresponds to the only noise case from the data model (i.e., $\mY = \frac{1}{\sqrt n} \mN \mZ$).

\section{DATASETS AND IMPLEMENTATION DETAILS}\label{appenx:datasets}
\begin{table}[t!]
\caption{Statistics of the datasets used in our experiments.}
\label{tab:stats}
\centering
\begin{sc}
\begin{tabular}{lccccc}
\toprule
Dataset & \#features & \#nodes & \#edges & \#classes \\
\midrule
Cora      & $1433$ & $2708$  & $5208$   & $7$  \\
CiteSeer  & $3703$ & $3327$  & $4552$   & $6$  \\
PubMed    & $500$  & $19717$ & $44338$  & $3$  \\
CORA-Full & $8710$ & $18703$ & $62421$  & $67$ \\
Photo     & $745$  & $7487$  & $119043$ & $8$ \\
CS & $6805$ & $18333$ & $81894$ & $15$ \\ 
\bottomrule
\end{tabular}
\end{sc}
\end{table}

Summary statistics of the datasets we used are shown in Table~\ref{tab:stats}. As mentioned in the main paper, we experiment on citation networks of Cora/CiteSeer/Pubmed/CoraFull \citep{SenNamataBilgicGetoorGalligherEliassi-Rad2008, shchur2018pitfalls}, as well as Amazon co-purchase networks of Photo and Co-author network of authors from Computer Science (CS) domain \citep{shchur2018pitfalls}. For train/valid/test splits, we follow the public split on Cora/CiteSeer/PubMed and construct the train/valid/test set on CoraFull/Photo/CS by randomly sampling $20$ nodes from each class to form the training set and $500$/$1000$ nodes respectively from the rest to form the validation and test set, as proposed in \citet{pmlr-v48-yanga16}.    

In line with our theoretical analysis, the main GCN architecture on which we experimented in this paper is a single-layer GCN (one iteration of message passing and update) stacked with a Multi-Layer Perception (MLP). The objective of the experiments is to validate our theoretical hypotheses and experiment with the robustness of GCN models under graph structure perturbation. 
We also study empirically to what extent the validated hypotheses extrapolate to scenarios where deeper GCN architectures with multiple layers of graph propagation are used and/or node features are also perturbed. In comparison to the state-of-the-art models, in particular to those which also place particular emphasis on the node features, we demonstrate that our proposed method has superior or comparable performance and can be further improved when
combined with other techniques. 

All the experiments are performed using the Adam optimiser \citep{kingma2014adam} and the same set of hyper-parameters, with learning rate being $1\text{e-}2$, number of epochs being $200$ and hidden feature dimension being $128$. We repeat each experiment $10$ times and report the resulting means and standard deviations to accurately report the impact of random initialisation.

We have made our implementation publicly available online\footnote{\href{https://github.com/ChangminWu/RobustGCN}{https://github.com/ChangminWu/RobustGCN}}.
It is built upon the open source library \textit{PyTorch Geometric} (PyG) under MIT license \citep{Fey/Lenssen/2019}. The experiments are run on a Intel(R) Xeon(R) W-2123 processor with $64$GB ram and a NVIDIA GeForce RTX 2080Ti GPU with $12$GB ram.

\section{ADDITIONAL EXPERIMENTS}\label{appenx:additionaL_experiments}
\subsection{Multiple Splits}\label{appenx:multi-split}
\begin{table}[t]
\caption{Performance of the GCN with and without node-feature kernel under perturbation on six real-world datasets with \textbf{multiple train/valid/test splits}. The format follows Table $1$ in the main paper.} \label{tab:multi-split}
\def\arraystretch{1.0}
\scriptsize
\begin{center}
\begin{tabular}{cc|cc|cc|cc|}
& & \multicolumn{2}{c|}{Cora} & \multicolumn{2}{c|}{CiteSeer} & \multicolumn{2}{c|}{PubMed}  \\
& ($\alpha$, $\beta$) & GCN & GCN-k & GCN & GCN-k & GCN & GCN-k \\
\hline
& (0.0, 0.0) & $ 76.25 \pm 2.32$ & $ 75.48 \pm 1.87$ & $ 66.31 \pm 2.04$ & $ 66.53 \pm 2.10$ & $ 74.80 \pm 2.07$ & $ 76.93 \pm 2.20$ \\

\multirow{2}{*}{Deletion} & (0.2, 0.0) & $ 74.82 \pm 1.68$ & $ 73.26 \pm 1.88$ & $ 64.96 \pm 1.57$ & $ 64.27 \pm 2.34$ & $ 75.14 \pm 2.06$ & $ 74.21 \pm 2.84$ \\

& (0.5, 0.0) & $ 70.78 \pm 1.53$ & $ 70.80 \pm 1.83$ & $ 63.14 \pm 1.49$ & $ 62.74 \pm 1.80$ & $ 74.17 \pm 1.75$ & $ 74.11 \pm 2.06$ \\ 

\multirow{2}{*}{Insertion} & (0.0, 0.5) & $ 67.01 \pm 2.11$ & $ \mathbf{71.96 \pm 1.79}$ & $ 57.35 \pm 2.02$ & $ \mathbf{62.19 \pm 1.68}$ & $ 63.70 \pm 2.95$ & $ \mathbf{71.62 \pm 2.01}$ \\

& (0.0, 1.0) & $ 58.92 \pm 2.29$ & $ \mathbf{68.50 \pm 1.45}$ & $ 50.53 \pm 2.13$ & $ \mathbf{59.60 \pm 2.04}$ & $ 57.07 \pm 1.94$ & $ \mathbf{69.56 \pm 1.92}$ \\ 
\end{tabular}\\
\vspace{0.05cm}
\begin{tabular}{cc|cc|cc|cc|}
& & \multicolumn{2}{c}{CoraFull} & \multicolumn{2}{c|}{Photo} & \multicolumn{2}{c|}{CS} \\
& ($\alpha$, $\beta$) & GCN & GCN-k & GCN & GCN-k & GCN & GCN-k \\
\hline

& (0.0, 0.0) & $ 58.42 \pm 2.12$ & $ 57.64 \pm 1.45$ & $ 91.48 \pm 0.94$ & $ 91.05 \pm 1.28$ & $ 92.09 \pm 0.98$ & $ 92.62 \pm 0.92$ \\

\multirow{2}{*}{Deletion} & (0.2, 0.0) & $ 56.36 \pm 1.70$ & $ 57.18 \pm 1.68$ & $ 90.51 \pm 1.29$ & $ 91.30 \pm 1.22$ & $ 91.37 \pm 1.26$ & $ 91.79 \pm 0.91$ \\
& (0.5, 0.0) & $ 54.58 \pm 1.33$ & $ 53.44 \pm 1.47$ & $ 90.07 \pm 1.63$ & $ 90.16 \pm 1.10$ & $ 89.78 \pm 1.04$ & $ 91.08 \pm 1.27$ \\

\multirow{2}{*}{Insertion} & (0.0, 0.5) & $ 48.53 \pm 1.58$	& $ \mathbf{53.55 \pm 1.44}$	& $ 81.77 \pm 2.31$ & $ 83.52 \pm 1.43$ & $ 88.19 \pm 0.96$ & $ \mathbf{91.25 \pm 1.00}$ \\
& (0.0, 1.0) & $ 43.18 \pm 1.73$ & $ \mathbf{50.77 \pm 1.89}$ & $ 75.21 \pm 4.30$ & $ 79.02 \pm 3.04$ & $ 83.83 \pm 1.60$	& $ \mathbf{90.32 \pm 0.89}$ \\

\end{tabular}
\end{center}
\end{table}

\citet{shchur2018pitfalls} argue that different train/valid/test splits of datasets may have a non-negligible impact on the performance of GNN models for the node classification task. To investigate the influence of different splits on our hypothesis, we construct train/valid/test split for each dataset following \citet{pmlr-v48-yanga16} over $10$ random seeds. As each experiment is repeated $10$ times, a total of $100$ results is obtained for a specific setting, i.e., specific $\alpha$, $\beta$ or $\epsilon$. Similar to the results of one split, we report the mean and standard deviations of the $100$ results in Table~\ref{tab:multi-split}. Although the standard deviation increases since we introduce more variation in the multi-split setting, the general trend remains the same, as shown in Table~\ref{tab:multi-split}. Our conclusion drawn in the main paper still holds: \emph{perturbations of the graph structure strongly influence the performance of the GCN and adding a node feature kernel can robustify the GCN against such perturbations}.

\subsection{Models beyond the GCN}\label{appenx:other-model}
\begin{table*}[t]
\caption{Performance of the GIN/GraphSage/GAT with and without node-feature kernel under perturbation on three citation datasets. The format follows Table $1$ in the main paper.} \label{tab:other-model}
\def\arraystretch{1.0}
\scriptsize
\begin{center}
\begin{tabular}{cc|cc|cc|cc|}
& & \multicolumn{6}{c|}{Cora}  \\
& ($\alpha$, $\beta$) & GIN & GIN-k & Sage & Sage-k & GAT & GAT-k \\
\hline

& (0.0, 0.0) & $78.56 \pm 1.12$	& $78.59 \pm 0.8$ &	$75.94 \pm 0.25$ & $\mathbf{77.12 \pm 0.65}$	& $78.20 \pm 0.54$ & $78.55 \pm 0.69$ \\

\multirow{2}{*}{Deletion} & (0.2, 0.0) & $76.61 \pm 0.58$ & $76.73 \pm 1.79$ & $73.87 \pm 0.92$ &	$75.29 \pm 1.25$ & $75.54 \pm 0.77$	& $\mathbf{77.01 \pm 0.30}$ \\

& (0.5, 0.0) & $71.31 \pm 1.06$ & $70.57 \pm 0.63$ & $67.35 \pm 0.92$ & $\mathbf{71.98 \pm 0.99}$ &	$71.25 \pm 0.89$ & $72.37 \pm 0.74$ \\

\multirow{2}{*}{Insertion} & (0.0, 0.5) &
$71.08 \pm 1.08$ & $72.19 \pm 0.87$	& $70.24 \pm 1.01$	& $71.42 \pm 1.23$ & $67.45 \pm 1.21$ &	$\mathbf{72.13 \pm 0.53}$ \\

& (0.0, 1.0) & $66.21 \pm 1.52$ & $67.96 \pm 0.67$ & $66.0 \pm 1.23$ & $\mathbf{70.57 \pm 0.95}$	& $62.14 \pm 1.46$ & $\mathbf{67.68 \pm 0.89}$ \\

\multirow{2}{*}{Delet.+Insert.} & (0.5, 0.5) & $56.82 \pm 1.35$ & $\mathbf{62.41 \pm 1.07}$ & $61.86 \pm 1.32$ & $63.61 \pm 0.94$ & $53.82 \pm 0.86$ & $\mathbf{61.04 \pm 1.24}$ \\

& (0.5, 1.0) & $51.20 \pm 2.04$ & $\mathbf{59.28 \pm 0.97}$ & $60.56 \pm 0.81$ & $\mathbf{64.13 \pm 0.52}$ & $46.97 \pm 1.15$ & $\mathbf{52.13 \pm 1.36}$ \\
\end{tabular}\\
\vspace{0.05cm}
\begin{tabular}{cc|cc|cc|cc|}
& & \multicolumn{6}{c|}{CiteSeer}  \\
& ($\alpha$, $\beta$) & GIN & GIN-k & Sage & Sage-k & GAT & GAT-k \\
\hline
& (0.0, 0.0) & $66.24 \pm 0.89$ & $66.96 \pm 0.86$ & $67.74 \pm 0.85$ & $68.97 \pm 0.60$ & $68.19 \pm 0.92$ & $67.82 \pm 0.87$ \\

\multirow{2}{*}{Deletion} & (0.2, 0.0) & $62.24 \pm 1.19$ & $\mathbf{66.70 \pm 1.44}$ & $65.97 \pm 1.06$ & $66.22 \pm 0.78$ & $66.31 \pm 0.93$ & $66.90 \pm 0.94$\\

& (0.5, 0.0) & $62.01 \pm 1.01$ & $62.30 \pm 1.24$ & $63.24 \pm 0.59$ & $64.97 \pm 1.03$ & $63.61 \pm 1.03$ & $65.43 \pm 0.83$ \\

\multirow{2}{*}{Insertion} & (0.0, 0.5) &
$58.90 \pm 1.31$ & $\mathbf{64.75 \pm 1.49}$ & $64.71 \pm 0.94$ & $66.21 \pm 0.72$ & $59.44 \pm 1.34$	& $\mathbf{62.61 \pm 1.30}$ \\

& (0.0, 1.0) & $54.61 \pm 1.28$ & $\mathbf{59.25 \pm 0.99}$ & $62.08 \pm 0.99$ & $63.04 \pm 1.07$	& $50.34 \pm 0.99$ & $\mathbf{58.46 \pm 0.98}$\\

\multirow{2}{*}{Delet.+Insert.} & (0.5, 0.5) & $50.46 \pm 2.02$	& $\mathbf{57.9 \pm 1.51}$ & $57.88 \pm 1.35$ & $\mathbf{63.56 \pm 0.67}$ & $50.75 \pm 1.21$ & $\mathbf{55.81 \pm 1.04}$ \\

& (0.5, 1.0) & $43.98 \pm 1.55$	& $\mathbf{48.4 \pm 1.47}$ & $58.51 \pm 1.13$ & $59.52 \pm 1.07$	& $43.56 \pm 1.19$ & $\mathbf{50.17 \pm 1.35}$ \\
\end{tabular}\\
\vspace{0.05cm}

\begin{tabular}{cc|cc|cc|cc|}
& & \multicolumn{6}{c|}{PubMed}  \\
& ($\alpha$, $\beta$) & GIN & GIN-k & Sage & Sage-k & GAT & GAT-k \\
\hline
& (0.0, 0.0) & $77.13 \pm 0.36$ & $77.02 \pm 0.58$ & $76.44 \pm 0.59$ & $77.09 \pm 0.64$ & $76.08 \pm 0.81$ & $76.63 \pm 0.47$\\

\multirow{2}{*}{Deletion} & (0.2, 0.0) & $75.96 \pm 0.47$ & $75.45 \pm 0.44$ & $75.38 \pm 0.34$ & $75.83 \pm 0.83$ & $75.93 \pm 0.46$ & $76.17 \pm 0.55$ \\

& (0.5, 0.0) & $74.29 \pm 0.79$ & $\mathbf{76.09 \pm 0.45}$ & $72.10 \pm 1.60$ & $\mathbf{76.24 \pm 0.49}$	& $73.46 \pm 0.45$ & $\mathbf{76.05 \pm 0.51}$ \\

\multirow{2}{*}{Insertion} & (0.0, 0.5) &
$71.85 \pm 0.79$ & $73.63 \pm 1.10$ & $73.48 \pm 0.48$ & $74.18 \pm 0.43$ & $67.42 \pm 0.63$	& $\mathbf{69.80 \pm 0.69}$\\

& (0.0, 1.0) & $64.27 \pm 1.60$ & $\mathbf{69.61 \pm 1.3}$ & $74.37 \pm 0.72$ & $74.44 \pm 0.32$ & $64.73 \pm 0.98$ & $65.38 \pm 0.73$\\

\multirow{2}{*}{Delet.+Insert.} & (0.5, 0.5) & $64.65 \pm 1.07$ & $\mathbf{68.24 \pm 0.61}$ & $72.59 \pm 0.61$ & $\mathbf{74.87 \pm 0.37}$ & $62.66 \pm 0.81$ & $\mathbf{66.51 \pm 0.9}$ \\

& (0.5, 1.0) & $59.12 \pm 1.19$	& $\mathbf{63.20 \pm 1.41}$	& $73.77 \pm 0.37$ & $72.64 \pm 0.45$ & $58.58 \pm 0.80$ & $\mathbf{63.60 \pm 1.27}$\\
\end{tabular}\\
\end{center}
\end{table*}

We also study the behaviour of our proposed method in a more general Message-Passing Neural Network (MPNN) setting beyond the GCN. We choose three characteristic models, which are GIN \citep{Xu2019}, GraphSage \citep{hamilton2017inductive} and GAT \citep{velivckovic2018graph} models, and observe their performance under graph structural perturbation with node feature kernel (our proposed method) on three citation datasets. The models are also implemented as a single graph-propagation layer followed by a MLP readout, as was the case for the GCN. Results are gathered in Table~\ref{tab:other-model}. Similar to the results of the GCN, we can observe from Table~\ref{tab:other-model} that on every model and every dataset, adding a node-feature kernel in the graph propagation helps to robustify the model performance against structural noises, in particular when edges are added. This empirical evidence demonstrates the versatility of our proposed method for a general MPNN model.  

\subsection{Deeper GCN architecture and Benchmark Models}\label{appenx:multi-layer}
\begin{table*}[t]
\caption{Performance of the GCN with and without node-feature kernel under perturbation on deep GCN models, compared with jump knowledge and GCNII. The format follows Table 1 in the main paper, where in addition we underline the second best result.} \label{tab:multi-layer_app}
\def\arraystretch{1.0}
\scriptsize
\begin{center}
\begin{tabular}{cc|c|c|c|c|c|c|}
 & &  \multicolumn{6}{c}{CiteSeer} \\
& ($\alpha$, $\beta$) & GCN & GCN-k ($\epsilon=0.5$) & GCN-k ($\epsilon=0.2$) & GCN-jk & GCNII & GCN-k-jk \\
\hline
& (0.0, 0.0) & $66.72 \pm 1.93$ & $66.76 \pm 0.78$ & $67.75 \pm 1.39$ & $68.70 \pm 1.06$ & $67.59 \pm 0.81$ & $69.10 \pm 0.92$ \\


Deletion& (0.5, 0.0) & $62.68 \pm 1.39$ & $\underline{64.52 \pm 2.03}$ & $62.72 \pm 1.65$ & $\mathbf{65.49 \pm 0.94}$ & $61.91 \pm 1.52$ & $63.82 \pm 2.01$\\


Insertion& (0.0, 1.0) & $45.14 \pm 1.89$ & $48.39 \pm 1.88$ & $56.69 \pm 1.91$ & $53.26 \pm 1.44$ & $\underline{57.41 \pm 1.67}$ & $\mathbf{62.83 \pm 1.06}$
 \\

Delet.+Insert.& (0.5, 0.5) & $39.68 \pm 1.96$ & $48.26 \pm 1.91$ & $51.10 \pm 1.26$ & $49.50 \pm 0.86$ & $\underline{53.19 \pm 1.58}$ & $\mathbf{58.12 \pm 1.02}$\\

\end{tabular}
\vspace{0.05cm}
\begin{tabular}{cc|c|c|c|c|c|c|}
 & &  \multicolumn{6}{c}{PubMed} \\
& ($\alpha$, $\beta$) & GCN & GCN-k ($\epsilon=0.5$) & GCN-k ($\epsilon=0.2$) & GCN-jk & GCNII & GCN-k-jk \\
\hline
& (0.0, 0.0) & $76.50 \pm 0.71$ & $77.82 \pm 0.86$ & $77.80 \pm 0.92$ & $77.36 \pm 0.74$ & $78.14 \pm 0.87$ & $77.56 \pm 1.62$\\


Deletion& (0.5, 0.0) & $73.92 \pm 0.64$	& $74.53 \pm 0.97$	& $\mathbf{76.19 \pm 0.73}$	& $73.00 \pm 0.93$ & $75.00 \pm 0.65$ & $\underline{75.60 \pm 1.09}$ \\


Insertion& (0.0, 1.0) & $46.98 \pm 2.87$ & $68.29 \pm 7.93$	& $72.42 \pm 1.11$	& $65.01 \pm 1.70$ & $\underline{72.86 \pm 0.75}$ & $\mathbf{73.20 \pm 0.89}$\\

Delet.+Insert.& (0.5, 0.5) & $62.35 \pm 1.01$ & $65.19 \pm 1.78$ & $\mathbf{70.79 \pm 1.08}$	& $65.07 \pm 0.83$	& $69.06 \pm 0.67$ & $\underline{66.28 \pm 1.37}$\\

\end{tabular}
\vspace{0.05cm}
\begin{tabular}{cc|c|c|c|c|c|c|}
& & \multicolumn{6}{c}{CoraFull} \\
& ($\alpha$, $\beta$) & GCN & GCN-k ($\epsilon=0.5$) & GCN-k ($\epsilon=0.2$) & GCN-jk & GCNII & GCN-k-jk \\
\hline
 & (0.0, 0.0) & $49.33 \pm 0.61$ & $55.27 \pm 0.97$ & $52.38 \pm 1.62$ & $\mathbf{58.56 \pm 1.20}$ & $\underline{56.83 \pm 0.67}$ & $56.10 \pm 0.89$ \\


Deletion& (0.5, 0.0) & $46.80 \pm 1.35$ & $45.33 \pm 1.37$ & $46.71 \pm 1.12$ & $51.42 \pm 1.48$	& $\underline{51.57 \pm 1.03}$ & $\mathbf{51.76 \pm 1.32}$
\\


Insertion& (0.0, 1.0) & $3.52 \pm 0.76$ & $9.60 \pm 2.62$ & $19.59 \pm 2.42$ & $42.07 \pm 0.92$ & $\underline{48.79 \pm 1.28}$	& $\mathbf{52.00 \pm 0.87}$ 
\\

Delet.+Insert.& (0.5, 0.5) & $7.10 \pm 1.60$ & $12.98 \pm 1.56$ & $22.06 \pm 1.27$ & $38.74 \pm 0.86$ & $\underline{41.27 \pm 1.31}$ & $\mathbf{45.74 \pm 1.22}$\\

\end{tabular}
\vspace{0.05cm}

\begin{tabular}{cc|c|c|c|c|c|c|}
 & &  \multicolumn{6}{c}{Photo} \\
& ($\alpha$, $\beta$) & GCN & GCN-k ($\epsilon=0.5$) & GCN-k ($\epsilon=0.2$) & GCN-jk & GCNII & GCN-k-jk \\
\hline
& (0.0, 0.0) & $87.67 \pm 1.59$ & $89.59 \pm 0.63$ & $88.64 \pm 1.19$ & $91.82 \pm 0.58$ & $\underline{92.05 \pm 0.88}$ & $\mathbf{92.42 \pm 0.58}$\\


Deletion& (0.5, 0.0) & $88.04 \pm 1.02$	& $88.81 \pm 0.33$ & $87.23 \pm 1.20$ & $\underline{90.44 \pm 1.07}$	& $87.65 \pm 1.85$ & $\mathbf{91.45 \pm 0.52}$\\


Insertion& (0.0, 1.0) & $27.75 \pm 3.39$ & $29.6 \pm 3.22$ & $24.05 \pm 2.56$ & $77.54 \pm 3.44$ & $\mathbf{85.68 \pm 1.94}$ & $\underline{81.89 \pm 3.24}$ \\

Delet.+Insert.& (0.5, 0.5) & $31.27 \pm 4.58$ & $27.65 \pm 6.01$ & $29.32 \pm 4.17$	& $\underline{75.38 \pm 5.41}$ & $\mathbf{87.32 \pm 1.11}$ & $73.57 \pm 5.92$\\

\end{tabular}

\end{center}
\end{table*}

In this set of experiments, we observe to what extent the conclusions drawn in our theoretical analysis carry over to deeper GCN architectures, and how our proposed model performs against similar methods in the deeper architecture setting. We add three extra message passing layer to the previous single-layer model and repeat our experiments on this deeper model as well as two state-of-the-art methods, the Jumping Knowledge (JK) and the GCN with residual connections and an identity mapping (GCNII), which are specifically designed for deeper models and take also advantage of the node feature information. Part of the results have already been shown in Section 4.1 of the main paper. In Table~\ref{tab:multi-layer_app} we show the results of the remaining four datasets, which agree  with the trends observed in the main paper. 

\subsection{Node Feature Noise}
We now provide further experiment results in Figure~\ref{fig:feature-noise} studying how node feature noise interacts with our proposed model. We observe that only when the node feature noise is five times larger than the graph structure noise, it starts to overshadow the benefit of adding the node feature kernel, as we can see from the pink curves, with-kernel (dashed or dotted line) performs worse than without-kernel (solid line).
Otherwise, the performance of our kernel is robust and matches the observed trend repeatedly demonstrated in the previous experiments. 

\begin{figure}[t!]
    \begin{center}
    \includegraphics[width=\linewidth]{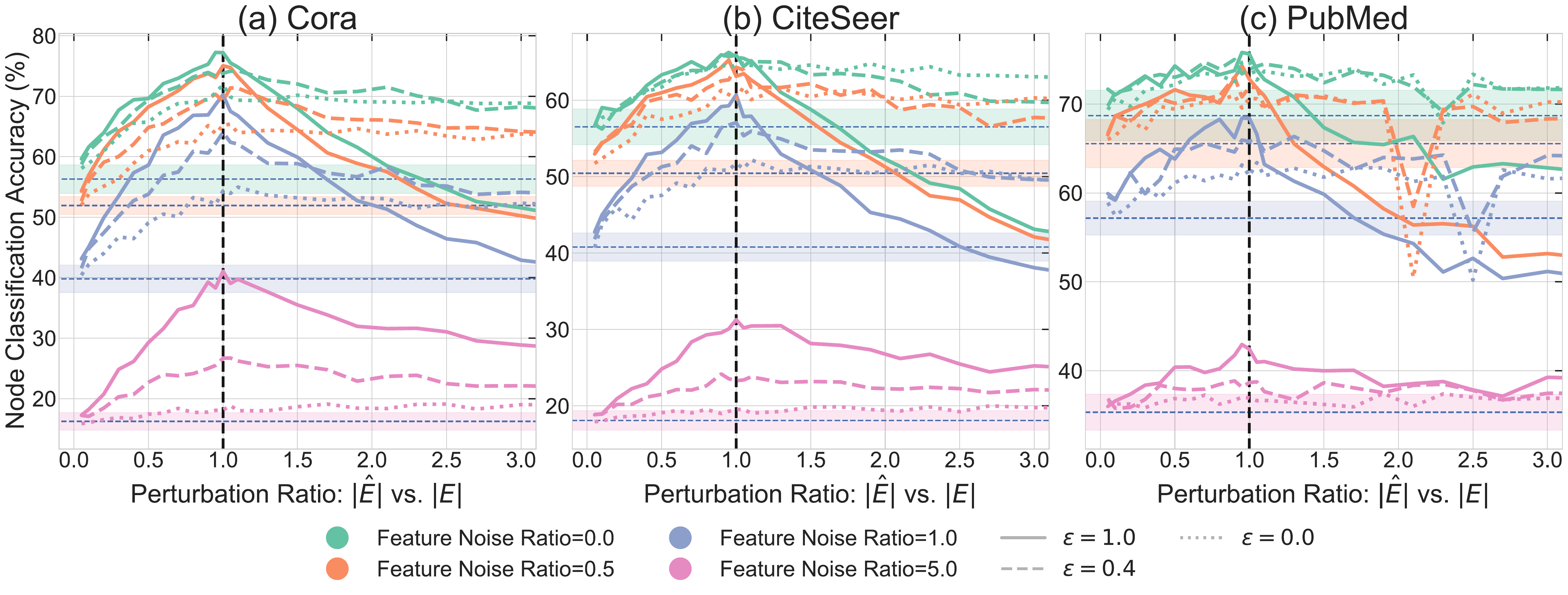}
    \caption{Experiment Results with the co-appearance of graph structural noise and node feature noise on three citation datasets. The vertical line at a rate of 1 represents the performance on the unperturbed graph. On its left is the \emph{edge deletion} case, with rate less than $1,$ where the most perturbed case corresponds to rate $0$; on the right is the \emph{edge insertion} case, where the perturbations grow with the rate. Different colours represent the extent of node feature perturbation.} 
    \label{fig:feature-noise}
    \end{center}
    \vspace{-0.5cm}
\end{figure}

\end{document}